\documentclass{article} %
\usepackage{iclr2027_conference,times}

\iclrfinalcopy %

\usepackage{amsmath,amsfonts,bm}

\def\eqref#1{equation~\ref{#1}}

\def\1{\bm{1}}

\DeclareMathAlphabet{\mathsfit}{\encodingdefault}{\sfdefault}{m}{sl}
\SetMathAlphabet{\mathsfit}{bold}{\encodingdefault}{\sfdefault}{bx}{n}

\newcommand{\R}{\mathbb{R}}

\usepackage{hyperref}
\usepackage{url}

\usepackage{amsmath,amssymb,amsthm,bm,mathtools}
\usepackage{graphicx,booktabs,array,xcolor}

\usepackage{hyperref}
\usepackage{url}
\usepackage{enumitem}
\usepackage{subcaption}
\hypersetup{hidelinks} %

\title{Explaining Hyperbolic Neural Networks \\via Geometry-Aware Relevance Propagation}

\author{Ping Xiong$^{1,2}$, Shanglin Li$^{1,2}$, Yi Ding$^{3}$, Thomas Schnake$^{4,5}$, Shinichi Nakajima$^{1,2,6}$\thanks{Correspondence to \texttt{nakajima@tu-berlin.de}.} \\
$^1$BIFOLD, Germany, $^2$Technische Universit\"at Berlin, Germany\\ $^3$Nanyang Technological University, Singapore, $^{4}$University of Toronto, Canada\\$^{5}$Vector Institute for Artificial Intelligence, Canada, $^{6}$RIKEN Center for AIP, Japan}

\newtheorem{proposition}{Proposition}
\newtheorem{definition}{Definition}

\newtheorem{corollary}{Corollary}
\newtheorem{theorem}{Theorem}
\newcommand{\one}{\bm 1}

\begin{document}
\maketitle

\lhead{}

\begin{abstract}
Hyperbolic neural networks introduce geometric operations that require explicit treatment in relevance propagation. Equivalent geometric realizations can produce different feature attributions, even when local relevance is conserved. We study this problem through Geometric Representation Invariance (GRI), a specialization of Implementation Invariance, and zero-curvature consistency, which requires identity relevance propagation when a geometric module approaches the identity. We propose LRP-radial-all for origin-centered radial modules, treating geometric scaling as modulation and assigning relevance entirely to the signal branch. The rule conserves relevance, is invariant to equivalent radial factorizations, and satisfies zero-curvature consistency, yielding GRI for a specified Poincar\'e--Lorentz logarithmic-map construction. In contrast, a conservative LRP-half baseline can violate both consistency criteria. Experiments on hyperbolic MNIST, sEEG, and CIFAR-10 classifiers assess attribution fidelity, qualitative explanations, and runtime. LRP-radial-all achieves competitive attribution fidelity across datasets with runtime comparable to Gradient$\times$Input and substantially lower than Integrated Gradients. These findings motivate geometry-aware propagation rules that distinguish relevance conservation from consistency across equivalent computations.
\end{abstract}

\section{Introduction}
\label{sec:intro}
Explainable artificial intelligence (XAI) seeks to clarify how models arrive at their predictions \citep{Gunning2019, DBLP:journals/pieee/SamekMLAM21, DBLP:conf/icml/HolzingerSMBS20, DBLP:series/lncs/11700}. However, most existing XAI methods have been developed primarily for conventional Euclidean neural networks. How these methods should be adapted to models whose computations are governed by non-Euclidean geometry remains much less understood.

This question becomes increasingly important for hyperbolic neural networks, which perform neural network operations in hyperbolic space and have attracted increasing attention \citep{peng2021hyperbolic, he2025hyperbolic}.
This interest is partly motivated by the mismatch between the exponential growth of tree-like hierarchies and the polynomial volume growth of Euclidean space. Hyperbolic space, by contrast, exhibits exponential volume growth with radius, making it well suited to representing hierarchical structures \citep{krioukov2010hyperbolic}.
Leveraging this geometric advantage, hyperbolic representations have shown promising performance across a range of tasks involving hierarchical structure, including computer vision \citep{DBLP:conf/nips/GaneaBH18, khrulkov2020hyperbolic, mettes2024hyperbolic}, natural language processing \citep{tifrea2018poincar}, foundation models \citep{desai2023hyperbolic,he2026helm}, and biosignal analysis \citep{zhou2026eeg,li2026heegnet}.

However, hyperbolic geometry poses distinct challenges for feature attribution: geometric maps couple feature coordinates, and equivalent representations can expose different computational factorizations of the same intrinsic function. An explanation method must therefore account for these operations while maintaining consistency across equivalent implementations.
Layer-wise relevance propagation (LRP) offers a particularly suitable framework for studying this problem, as it redistributes a prediction backward through individual network modules using explicit local rules \citep{bach2015pixel}. This modular structure allows us to examine how relevance should propagate through hyperbolic operations and whether local conservation is compatible with Geometric Representation Invariance.

In this work, we propose a relevance propagation framework for origin-centered hyperbolic operations. 
The origin provides a natural fixed reference point whose tangent space admits a Euclidean representation and is mathematically simple, facilitating efficient computation \citep{liu2019hyperbolic}.
We formulate LRP-radial-all rule for logarithmic maps and exponential maps, treating geometric scaling as modulation and assigning relevance to the signal branch. We study Geometric Representation Invariance (illustrated in Figure~\ref{fig:overview}) as a specialization of Implementation Invariance \citep{DBLP:conf/icml/SundararajanTY17}, prove consistency for the specified logarithmic-map construction, and show that a conservative LRP-half baseline can violate this property under equivalent radial factorizations. 
We evaluate the framework through controlled simulations and extensive experiments on various datasets and model architectures, including image and time series.
Among the five equivalent hyperbolic models \citep{cannon1997hyperbolic}, Poincaré and Lorentz are the most popular choices and are covered in this paper \citep{peng2021hyperbolic, he2025hyperbolic}.

We make three main contributions.
\begin{itemize}[leftmargin=1.35em,itemsep=0.2em,topsep=0.2em]
    \item We propose \emph{LRP-radial-all} for origin-centered hyperbolic modules, assigning relevance to the signal while treating geometric scaling as modulation.
    \item We specialize Implementation Invariance to \emph{GRI} and introduce \emph{zero-curvature consistency}. We prove conservation, radial factorization invariance, and zero-curvature consistency for radial-all, including GRI for the specified logarithmic-map test, and show that conservative LRP-half can violate both consistency criteria.
    \item We evaluate attribution fidelity and runtime on MNIST, sEEG, and CIFAR-10 hyperbolic classifiers, alongside controlled consistency tests.
\end{itemize}

\begin{figure}[t]
\centering
\includegraphics[width=0.9\linewidth]{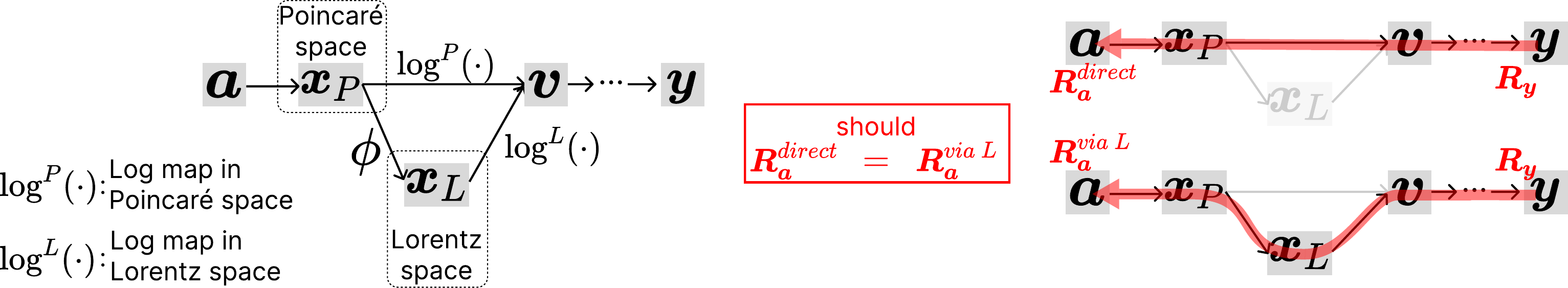}
\caption{Geometric Representation Invariance (GRI) requires explanations at the same input after equivalent internal coordinate realizations to be the same.}
\label{fig:overview}
\end{figure}

\section{Related Work}
\label{sec:related}
\noindent\textbf{Hyperbolic neural networks.}
Early work on hyperbolic neural networks extended standard neural operations to hyperbolic geometry. \citet{DBLP:conf/nips/GaneaBH18} introduced hyperbolic formulations of multinomial logistic regression and feed-forward layers using the Poincar\'e ball. Subsequent work broadened the architectural scope: \citet{shimizu2020hyperbolic} developed hyperbolic fully connected, convolutional, and attention layers, while \citet{DBLP:conf/iclr/BdeirSL24} introduced fully hyperbolic convolutional networks in the Lorentz model for computer vision. 
Beyond vision, hyperbolic architectures have also been applied to intracranial recordings. For example, \citet{li2026heegnet} incorporated hyperbolic layers into an EEGNet-based model for working-memory load classification from sEEG signals.
Together, these developments have made geometry-specific operations such as exponential maps, logarithmic maps, and manifold-valued feature transformations fundamental building blocks of modern hyperbolic architectures.

\noindent\textbf{Explainability for HNN.}
In contrast to the rapid development of hyperbolic architectures, their explainability has received comparatively limited attention.
Most feature-attribution methods were originally formulated for models operating in Euclidean spaces, and therefore do not explicitly account for the geometric operations introduced by HNNs.
Recent work has begun to incorporate non-Euclidean geometry into explainability.
Manifold Integrated Gradients (MIG)
\citep{zaher2024manifold} replaces the straight-line path of Integrated Gradients with geodesic paths on a learned Riemannian data manifold, thereby aligning attribution with the intrinsic geometry of the data.
Diffeomorphic counterfactual methods use generative models to construct coordinate systems in which gradient-based search produces semantically meaningful changes to the input \citep{DBLP:journals/pami/DombrowskiGMK24}. 
Such approaches demonstrate the relevance of geometry to explainability, but focus primarily on the geometry of the data manifold, rather than on propagating explanations through the internal geometric operations of a hyperbolic neural network.
Our work addresses the complementary problem.
We study how relevance should be propagated through the geometric primitives of HNNs themselves, with particular attention to logarithmic and exponential maps and transformations between equivalent hyperbolic representations.

\noindent\textbf{Layer-wise relevance propagation (LRP).}
\label{sec:RW.LRP}
Layer-wise relevance propagation (LRP) \citep{bach2015pixel} decomposes a model prediction into input relevance scores through architecture-specific backward redistribution rules, typically at a computational cost comparable to a backward computation. Deep Taylor Decomposition provides a theoretical foundation for certain LRP rules through local Taylor expansions of relevance functions \citep{DBLP:journals/pr/MontavonLBSM17, DBLP:journals/pieee/SamekMLAM21}.
Our work is particularly related to the rules for multiplicative interactions.
For gated recurrent networks, \citet{arras2017recurrent} assign all relevance to the source branch and none to the gate, interpreting the latter as a modulation of the signal. Similarly, \citet{DBLP:conf/icml/AliSEMMW22} treat attention weights as fixed during relevance propagation and route relevance through the value branch, following an LRP-all-style allocation for the attention-value product.
Alternatively, LRP-half for LSTM \citep{DBLP:series/lncs/ArrasAWMGMHS19} propagates half of the relevance to weight and half to signal. AttnLRP \citep{achtibat_attnlrp_2024} derives a uniform redistribution rule for multiplicative interactions.
These approaches motivate different treatments of hyperbolic radial maps.

\section{Geometry-aware Criteria}
\label{sec:geometry_requirements}

\subsection{Hyperbolic Neural Networks}
\label{sec:hnn}

Hyperbolic neural networks (HNNs) generalize neural operations
to hyperbolic space \citep{DBLP:conf/nips/GaneaBH18}.
We describe a representative architecture in the Poincar\'e ball
$\mathbb D_c^d=\{\bm x\in\mathbb R^d:
c\|\bm x\|_2^2<1\}$ with curvature $-c$, where $c>0$.
An input feature vector $\bm a\in\mathbb R^d$ is first
interpreted as an origin-tangent vector and mapped to the ball:
$    \bm h^{(0)}=\exp_{\bm0}^{c}(\bm a).
$
Each hidden layer performs a hyperbolic linear
transformation, bias addition, and activation:
\begin{align}
    \textstyle\bm m^{(\ell)}
    &=
    \exp_{\bm0}^{c}\!(
        \bm W^{(\ell)}
        \log_{\bm0}^{c}(\bm h^{(\ell-1)})
    )\oplus_c
    \exp_{\bm0}^{c}(\bm b^{(\ell)}),\\
    \textstyle\bm h^{(\ell)}
    &=
    \exp_{\bm0}^{c}\!(
        \sigma\!(\log_{\bm0}^{c}(\bm m^{(\ell)}))
    ),
\end{align}
where $\bm W^{(\ell)}$ and $\bm b^{(\ell)}$ are trainable
parameters, $\oplus_c$ denotes M\"obius addition (definition see Eq.\ref{eq:mobius-addition})
and $\sigma$ is an element-wise activation applied in the tangent space. We show the equations of $\exp_{\bm0}^{c}(\cdot)$ and $\log_{\bm0}^{c}(\cdot)$ in Figure~\ref{fig:split}.
For a tangent-space readout, the final representation is mapped back to Euclidean coordinates and passed to a prediction
head $g$:
$    \hat{\bm y}
    =
    g\!\left(\log_{\bm0}^{c}(\bm h^{(L)})\right).
$
Appendix~\ref{app:lorentz} summarizes the Lorentz model and its correspondence with the Poincar\'e representation.

\subsection{Geometric Representation Invariance}
\label{sec:GRI}

The Poincar\'e ball and Lorentz hyperboloid are isometric representations of the same hyperbolic geometry.
By transforming intermediate representations and the associated operations consistently, a model can be expressed in either space without changing its input-output function.
Ideally, explanations in the original input space should also remain unchanged under such a conversion.
This requirement is motivated by Implementation Invariance \citep{DBLP:conf/icml/SundararajanTY17}, which states that functionally equivalent models should yield identical input attributions.
We specialize this principle to equivalent hyperbolic
representations and refer to the resulting criterion as
\emph{Geometric Representation Invariance} (GRI).

\begin{definition}[GRI as a restricted Implementation Invariance criterion]
\label{def:gri}
Let $\bm a\in\R^p$ be the original feature vector. Let $F_P$ and $F_L$
implement the same scalar target function on a common input domain, with
internal hyperbolic computations related by the specified isometries.
For a fixed explanation procedure, GRI requires
\begin{equation}
 \mathcal E(F_P,\bm a)=\mathcal E(F_L,\bm a),
 \label{eq:gri}
\end{equation}
where $\mathcal E$ denotes an explanation method.
\end{definition}

For gradient-based methods defined solely through evaluations of the input-output function such as Gradient$\times$Input \citep{ancona2018towards} and Integrated Gradients \citep{DBLP:conf/icml/SundararajanTY17}, GRI follows directly from functional equivalence.
For propagation-based methods such as LRP, however, GRI is non-trivial because relevance redistribution depends on intermediate computations and local propagation rules.
Equivalent geometric representations may expose different computational factorizations, so local relevance conservation alone does not guarantee identical input attributions.

\noindent\textbf{A controlled test of GRI.}
\label{sec:gri-test}
Figure~\ref{fig:overview} illustrates a controlled test of whether relevance propagation depends on the internal representation of a hyperbolic logarithmic map. Starting from the same original input $\bm a$, a shared encoder produces a Poincar\'e representation $\bm x_P\in\mathbb D_c^d$. We compare two computational paths: the direct path applies the Poincar\'e logarithmic map, whereas the alternative path first converts $\bm x_P$ to the corresponding Lorentz point $\bm x_L=\phi(\bm x_P)$ and then applies the Lorentz logarithmic map. 

After aligning the tangent coordinates, both paths produce the same vector $\bm v\in\mathbb R^d$:
\begin{equation}
    \bm v
    =
    \alpha_P(\bm x_P)\bm x_P
    =
    \alpha_L(\bm x_L)\bar{\bm x}_L,
    \qquad
    \bar{\bm x}_L
    =
    \gamma(\bm x_P)\bm x_P.
    \label{eq:aligned-log}
\end{equation}
Here $\bm x_L=(x_{L,0},\bar{\bm x}_L)$, and $\alpha_L$ includes the factor $1/2$ that converts the ambient Lorentz tangent vector $(0,2\bm v)$ to the common coordinates $\bm v$.
The coordinate conversion and aligned logarithmic maps are derived in Appendix~\ref{app:log}.
Both paths then use the same downstream computation $y=G(\bm v)$, so they implement the same input-output function.

To test GRI, we initialize relevance at the same scalar output $y$ and propagate it backward along each path, as indicated by the red arrows in Figure~\ref{fig:overview}.
The downstream propagation is identical and therefore gives the same relevance $\bm R_{\bm v}$ to both realizations.
The direct path propagates this relevance through the Poincar\'e logarithmic map and the alternative path propagates it through the aligned Lorentz logarithmic map and the coordinate
conversion $\phi$.
Using the same relevance rule through the shared encoder, we compare the resulting explanations at the original input:
\begin{equation}
    \bm R_{\bm a}^{\mathrm{direct}}
    \stackrel{?}{=}
    \bm R_{\bm a}^{\mathrm{via}\,L}.
    \label{eq:log-gri-test}
\end{equation}
This construction isolates the effect of an equivalent geometric realization while keeping the input, prediction, and surrounding computation unchanged. 

Appendix~\ref{app:exp} presents a controlled exponential-map test with a shared Lorentz downstream network and an explicitly matched treatment of time-coordinate relevance.

\subsection{Zero-curvature consistency}
\label{sec:zero_curvature}

A natural consistency requirement arises when hyperbolic operations approach their Euclidean counterparts as curvature vanishes \citep{DBLP:conf/nips/GaneaBH18}. In particular, the origin-centered Poincar\'e exponential and logarithmic maps converge to the identity in the coordinate convention used here. Once such a geometric module becomes an identity transformation, it should no longer redistribute relevance among input coordinates. This motivates requiring its backward relevance rule to approach identity propagation as well. We formalize this local requirement as zero-curvature consistency.

\begin{definition}[Zero-curvature consistency]
\label{def:zero-curvature}
Let $f_c$ be a family of geometric modules expressed in fixed coordinates such that $f_c(\bm x)\to k\bm x$ as $c\to0$, where $k>0$ is fixed and independent of $\bm x$. We use identity relevance propagation for the limiting constant scaling, treating $k$ as a fixed modulation factor. Let $\mathcal R_c(\bm x,\bm R)$ denote the input relevance obtained by propagating a fixed incoming relevance vector $\bm R$ through $f_c$. The propagation rule is zero-curvature consistent if
\begin{equation}
\lim_{c\to0}\mathcal R_c(\bm x,\bm R)=\bm R
\end{equation}
for every input $\bm x$ and incoming relevance $\bm R$.
\end{definition}

This definition includes both identity limits and fixed coordinate scalings. The origin Poincar\'e maps approach the identity, whereas the aligned Lorentz spatial logarithmic and exponential maps approach scaling by $1/2$ and $2$, respectively. These constant factors reflect the tangent-coordinate convention and receive no separate relevance.

\section{Relevance Propagation Rules}
\label{sec:method}

In this section, we define relevance propagation rules for origin-centered hyperbolic modules. We introduce LRP-radial-all and contrast it with a conservative LRP-half baseline, showing that conservation alone does not ensure consistency across equivalent radial factorizations. 
For tangent-space linear layers, we use standard Euclidean LRP rules and their propagation formulas and conservation conditions are summarized in Appendix~\ref{app:numerics}. For M\"obius bias addition, we adopt a separate signal-only convention (see Appendix~\ref{app:bias}).

\subsection{Consistent relevance propagation through geometric modulation}
\label{sec:multiplicative-attribution}

Conservation alone does not determine how relevance should be allocated between multiplicative branches. Symmetric splitting can depend on the grouping of multiplication operations, while grouping-invariant symmetric rules require additional assumptions and may have restricted domains (Appendix~\ref{app:multiplicative-classification}). For our geometric modules, we instead distinguish the feature signal from its scalar modulation and adopt an all-to-signal allocation.

\noindent\textbf{Geometry as a modulation branch.}
For a radial map $\bm y=\alpha(\bm x)\bm x$, the two multiplicative branches have different roles: $\bm x$ carries the feature coordinates, while $\alpha(\bm x)$ applies a shared geometric scaling. We adopt a signal-based attribution convention that preserves the incoming feature-wise relevance allocation across this modulation. This choice is motivated by two properties. First, splitting the geometric scaling into successive modulation factors should not change the relevance propagated along the designated signal path. Second, when an origin map approaches the identity in the zero-curvature limit, its propagation rule should approach identity propagation as well.

These requirements motivate an all-to-signal allocation, consistent with earlier treatments of source-gate interactions \citep{arras2017recurrent} and attention-value products \citep{DBLP:conf/icml/AliSEMMW22}. 

\begin{definition}[LRP-radial-all]
\label{def:radial-all}
For a scalar-signal module $\bm y=\alpha(\bm x)\bm s$, LRP-radial-all assigns all incoming relevance to the designated signal $\bm s$ and none to the modulation factor:
\begin{equation}
R_\alpha=0, \qquad \bm R_{\bm s}=\bm R_{\bm y}.
\label{eq:geometry-weight-rule}
\end{equation}
For a radial map of a single vector, $\bm s=\bm x$ and $\alpha$ depends only on its radius.
\end{definition}

\begin{proposition}[Conservation and signal-path factorization invariance]
\label{prop:radial}
Consider functionally equivalent scalar-signal realizations that differ only in the splitting or grouping of scalar modulation factors along a designated signal path. Suppose the signal dimension and coordinate order are preserved, and the incoming relevance is identical. Applying the all-to-signal rule at every multiplication yields $\bm R_{\bm x}=\bm R_{\bm y}$, independently of the number or grouping of modulation factors, and conserves signed total relevance.
\end{proposition}

\begin{proof}
Each multiplication has backward operator $I_d$ on its designated signal and assigns zero relevance to its modulation branch. A chain of $m\geq1$ such modules therefore gives $\bm R_{\bm x}=I_d^m\bm R_{\bm y}=\bm R_{\bm y}$, and consequently $\one^\top\bm R_{\bm x}=\one^\top\bm R_{\bm y}$.
\end{proof}

The proposition provides the local consistency used in our controlled GRI test.
More generally, necessary conditions for GRI and zero-curvature consistency can be derived (see Appendix~\ref{app:radial-characterization}) for the fixed-proportion radial rules
\begin{equation}
\textstyle R_{x_i}=\eta R_{y_i}+(1-\eta)\frac{x_i^2}{\|\bm x\|_2^2}\sum_jR_{y_j}, \qquad \eta\in[0,1],
\label{eq:eta-radial-family}
\end{equation}
where the same constant $\eta$ is used at every module and the modulation relevance is redistributed through the squared norm. This family includes LRP-half at $\eta=1/2$ and LRP-radial-all at $\eta=1$. For $d\geq2$, invariance under equivalent nonzero radial factorizations permits only $\eta=0$ or $\eta=1$. Requiring zero-curvature consistency as defined in Definition~\ref{def:zero-curvature} further selects $\eta=1$. Thus, within this specified family, radial-all is the unique rule satisfying both criteria.

\begin{figure}[t]
\centering
\includegraphics[width=0.7\linewidth]{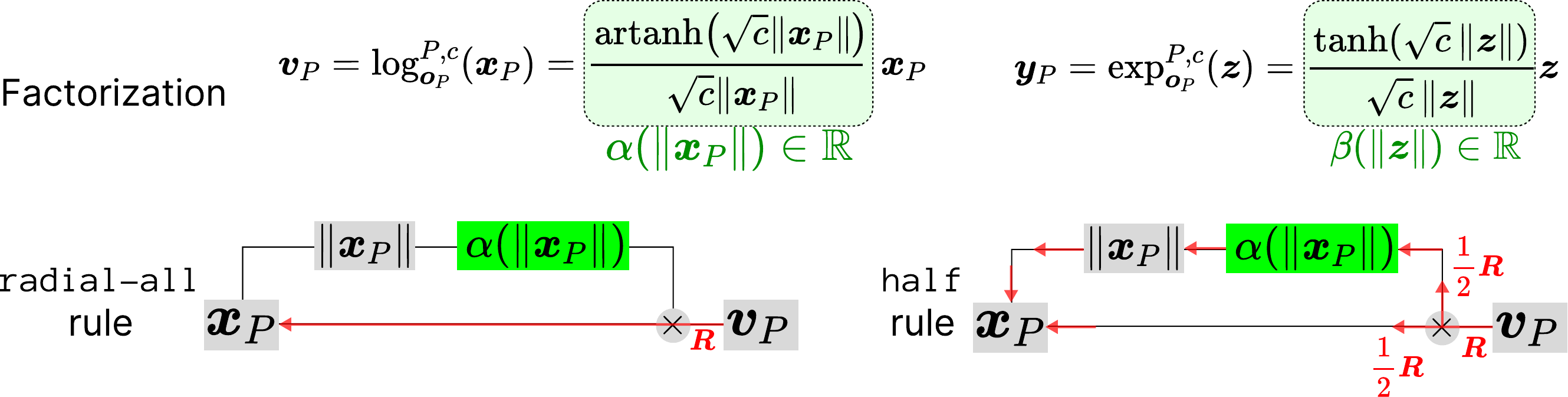}
\caption{Top: factorization of $\log$ and $\exp$ map with origin point in Poincar\'e space. Bottom: our proposed radial-all rule and half rule for LRP applied to such layer.}
\label{fig:split}
\end{figure}

\subsection{LRP-half as an alternative}

Motivated by equal-split rules for multiplicative interactions \citep{DBLP:series/lncs/ArrasAWMGMHS19, achtibat_attnlrp_2024}, we consider an LRP-half baseline that assigns half of the relevance to the signal and half to the radial factor, redistributing the latter through the squared norm (Figures~\ref{fig:split} and~\ref{fig:alpha_x}). In the logarithmic-map test of Eq.~\ref{eq:aligned-log}, the direct Poincar\'e path contains one radial module, whereas the equivalent Lorentz path contains two. For $\bm x_P\neq\bm0$, their backward allocations are
\begin{equation}
\textstyle R_{x_{P,i}}^{P}=\frac12R_{v_i}+\frac12\frac{x_{P,i}^2}{\|\bm x_P\|_2^2}\sum_jR_{v_j}, \qquad
R_{x_{P,i}}^{L}=\frac14R_{v_i}+\frac34\frac{x_{P,i}^2}{\|\bm x_P\|_2^2}\sum_jR_{v_j}. \label{eq:half-poincare-main}
\end{equation}
Both conserve the incoming total relevance, but generally disagree coordinate-wise. With an identity encoder, this difference directly violates GRI at the common original input. Appendix~\ref{app:half} provides the derivation and an explicit counterexample. Appendix~\ref{sec:exp_geometry} reports a numerical test.

\noindent\textbf{Failure of zero-curvature consistency.}
\label{sec:violates_consistency}
Although $\log_{\bm0}^{P,c}(\bm x_P)\to\bm x_P$ as $c\to0$, Eq.~\ref{eq:half-poincare-main} is independent of curvature for fixed input and incoming relevance, and generally differs from identity propagation. Thus, LRP-half also fails zero-curvature consistency and the same argument also applies to the origin exponential map. LRP-radial-all, by contrast, has identity backward action for every curvature. Conservation alone therefore guarantees neither of the two consistency criteria.

\section{Experiments}
\label{sec:experiments}

We compare attribution methods on hyperbolic classifiers for MNIST, sEEG, and CIFAR-10 using sparsity-fidelity metrics and qualitative visualizations. LRP-radial-all achieves competitive attribution fidelity. Runtime comparisons further show that LRP-radial-all is comparable in efficiency to Gradient$\times$Input and substantially faster than Integrated Gradients. We also present a controlled GRI numerical test in Appendix~\ref{sec:exp_geometry}.

\subsection{MNIST}
\label{sec:exp_mnist}

\noindent\textbf{Model and training.}
We use a Poincar\'e hyperbolic MLP with dimensions $784$-$128$-$10$, fixed curvature $-1$, M\"obius biases, and a tangent-space ReLU between the two hyperbolic linear layers. Images have pixel values in $[0,1]$ and flattened inputs are scaled by $0.1$ and mapped to the manifold using the origin exponential map. An origin logarithmic map produces the output logits.
The training configuration uses $55000$ training and $5000$ validation images, Adam with learning rate $10^{-3}$, batch size $128$ and cross-entropy loss. Training runs for $100$ epochs, with checkpoint selection by validation accuracy. The model has $97.44\%$ test accuracy.

\noindent\textbf{Explanation methods.}
All quantitative comparisons explain the original predicted-class logit of the model. We compare radial-all and the half rule for hyperbolic mapping layers using identical linear LRP-$\gamma$ rules ($\gamma=0.25$, $\epsilon=10^{-9}$) for other layers. Fixed bias branches receive zero relevance.
Baselines are Gradient$\times$Input, Integrated Gradients (IG), Patch
Occlusion, and random. IG uses a black-image baseline and $512$ midpoint integration steps along a straight path in pixel space. Patch Occlusion uses black $4\times4$ windows with stride $4$, assigning the target-logit drop to each pixel in the window.

\noindent\textbf{Sparsity-fidelity protocol.}
We evaluate randomly sampled $1000$ images from test set. In one image, pixels are ranked by their signed attribution scores in descending order. Let $S_k$ contain the top $k$ pixels and let $m_{S_k}$ be their binary pixel mask. For the fixed original predicted class $t$, we define
\begin{align}
 \operatorname{Sparsity}(k)&=1-k/K, &s_t(x)=\operatorname{softmax}(f(x))_t, \label{eq:sparsity}\\
 \operatorname{Fidelity}^{+}(k)&=s_t(x)-s_t(x\odot(1-m_{S_k})), &\operatorname{Fidelity}^{-}(k)=s_t(x)-s_t(x\odot m_{S_k}),\label{eq:fidplus}
\end{align}
where $K$ means the total number of pixels.
Higher Fidelity$^{+}$ indicates greater necessity of selected regions and lower Fidelity$^{-}$ indicates greater sufficiency. 
We use 50 regions in sparsity, and compute per-image AUC by  integration over actual sparsity values. Random rankings are averaged over $10$ repetitions per image. We report mean AUCs and percentile $95\%$ confidence intervals from $1000$ image-level bootstrap resamples.

\begin{table*}[t]
\centering
\small
\setlength{\tabcolsep}{3pt}
\caption{Sparsity-fidelity AUC on MNIST, sEEG and CIFAR-10. $F^{+}$ and $F^{-}$ denote Fidelity$^{+}$ and Fidelity$^{-}$ AUC.
Subscripts report $\pm h$, where $h=\max(\mu-L,U-\mu)$ for the 95\% bootstrap interval $[L,U]$.
Bold indicates the best mean within each dataset and metric. 
}
\label{tab:mnist_seeg_cifar_fidelity}
{\small
\begin{tabular}{lcccccc}
\toprule
& \multicolumn{2}{c}{MNIST}
& \multicolumn{2}{c}{sEEG}
& \multicolumn{2}{c}{CIFAR-10} \\
\cmidrule(lr){2-3}\cmidrule(lr){4-5}\cmidrule(lr){6-7}
Method & $F^{+}\uparrow$ & $F^{-}\downarrow$ & $F^{+}\uparrow$ & $F^{-}\downarrow$ & $F^{+}\uparrow$ & $F^{-}\downarrow$ \\
\midrule
LRP$_{\text{radial-all}}$
& $\mathbf{0.670}_{\pm0.005}$ & $\mathbf{-0.003}_{\pm0.005}$
& $\mathbf{0.313}_{\pm0.096}$ & $-0.111_{\pm0.017}$
& $0.637_{\pm0.020}$ & $\mathbf{0.277}_{\pm0.019}$ \\
LRP$_{\text{half}}$
& $0.562_{\pm0.011}$ & $0.063_{\pm0.008}$
& $0.310_{\pm0.095}$ & $-0.106_{\pm0.017}$
& $0.588_{\pm0.022}$ & $0.505_{\pm0.022}$ \\
\hline
Gradient$\times$Input
& $0.410_{\pm0.017}$ & $0.202_{\pm0.014}$
& $0.303_{\pm0.096}$ & $\mathbf{-0.113}_{\pm0.017}$
& $\mathbf{0.703}_{\pm0.020}$ & $0.547_{\pm0.023}$ \\
Integrated Grad.
& $0.664_{\pm0.005}$ & $0.000_{\pm0.005}$
& $0.309_{\pm0.096}$ & $\mathbf{-0.113}_{\pm0.017}$
& $0.683_{\pm0.020}$ & $0.495_{\pm0.022}$ \\
Patch Occlusion
& $0.504_{\pm0.013}$ & $0.067_{\pm0.007}$
& - & - & - & - \\
Random
& $0.193_{\pm0.005}$ & $0.193_{\pm0.005}$
& $0.057_{\pm0.023}$ & $0.058_{\pm0.022}$
& $0.664_{\pm0.022}$ & $0.663_{\pm0.022}$ \\
\bottomrule
\end{tabular}
}
\end{table*}

\begin{table}[t]
\centering
\caption{Runtime comparison of explanation methods.  We average the runtime on 3 runs. }
{\small
\begin{tabular}{lcccc}
\toprule
Dataset
& LRP-radial-all
& Gradient$\times$Input
& Occlusion
& Integrated Gradients \\
\midrule
MNIST (ms/image)
& {0.505}
& 0.960
& 12.676
& 41.511 \\
sEEG (ms/sample)
& 5.236
& {3.813}
& -
& 94.431\\
CIFAR-10 (ms/image)
& 15.688
& 12.730
& -
& 66.375\\
\bottomrule
\end{tabular}
}
\label{tab:runtime}
\end{table}

\noindent\textbf{Results.}
As shown in Table~\ref{tab:mnist_seeg_cifar_fidelity} and Figure~\ref{fig:mnist_quantitative}, LRP-radial-all achieves the highest Fidelity$^+$ AUC and the lowest Fidelity$^-$ AUC among the evaluated methods. Although the improvement from IG is marginal, LRP-radial-all is substantially faster than IG as in Table~\ref{tab:runtime}.
Figure~\ref{fig:mnist_qualitative} shows two example heatmaps by our method and baselines. LRP-radial-all highlights the important curves and intersections in the number pictures. Integrated Gradients can capture the important parts as well, but the heatmaps appear noisier.

\begin{figure}[t]
    \centering
    \includegraphics[width=0.8\linewidth]{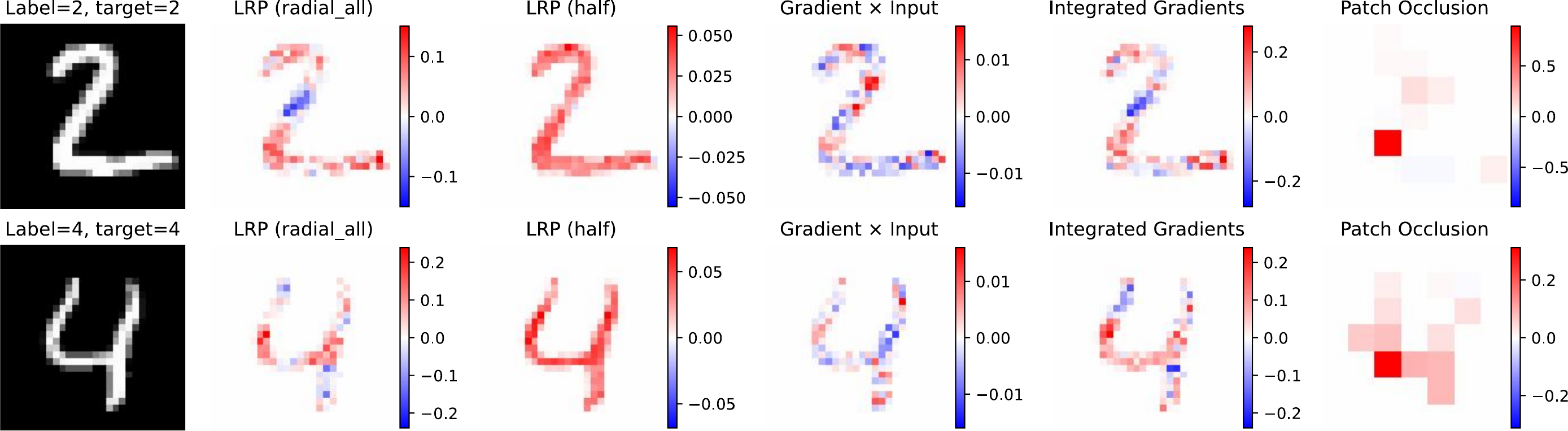}
    \caption{Heatmaps on images of MNIST. }
    \label{fig:mnist_qualitative}
\end{figure}

\subsection{sEEG}
sEEG captures interactions across electrodes and temporal scales, giving rise to hierarchical structure that can be naturally represented in hyperbolic space. 
Recent hyperbolic approaches to sEEG modeling have achieved state-of-the-art performance and substantially outperformed their Euclidean counterparts \citep{guillemaud2025hyperbolic, li2026heegnet}. 
Following \citet{li2026heegnet}, we evaluate on the working-memory sEEG dataset introduced by \citet{boran2020dataset}. The dataset contains 9 subjects and many trials for each subject.
Each trial is represented as a multichannel time series $\mathbf{X}\in\mathbb{R}^{C\times T}$, where $C$ denotes the number of electrodes and $T$ the number of temporal samples, and is associated with a binary label indicating whether the working-memory load is high or low. 
Following \citet{li2026heegnet}, we trained a variant of EEGNet in which the multi-layer perceptron is replaced with its hyperbolic counterpart.
On each subject we trained the model using 5-fold cross-validation.

\noindent\textbf{Qualitative results.}
We take subject 2 as an example and plot the signals and LRP-radial-all heatmaps averaged on all samples of it by the label in Figure~\ref{fig:seeg_visualization}. In the low-load condition, a prominent positive relevance region is concentrated within a relatively limited subset of channels, indicating a localized contribution to the model prediction. In contrast, the high-load condition exhibits more distributed relevance across multiple channels during the middle portion of the time window. Stronger relevance concentrations also emerge toward the later portion of the high-load window, indicating a greater contribution from these temporal segments. We validate the importance of the relevant time windows to the prediction in Appendix~\ref{app:seeg_validation}. In addition, matched upper-channel regions show opposite relevance polarities between the two conditions, indicating a class-dependent difference in how the model uses the same spatiotemporal region. The annotated boxes highlight these representative attribution patterns. The more distributed relevance observed under the high-load condition is consistent with recent intracranial evidence~\citep{Yang2025} showing increased inter-regional information sharing as working-memory demands increase. 

\begin{figure}
    \centering
    \begin{subfigure}{0.49\textwidth}
        \centering
        \includegraphics[width=\linewidth, trim=0 22.45cm 0 3cm, clip]{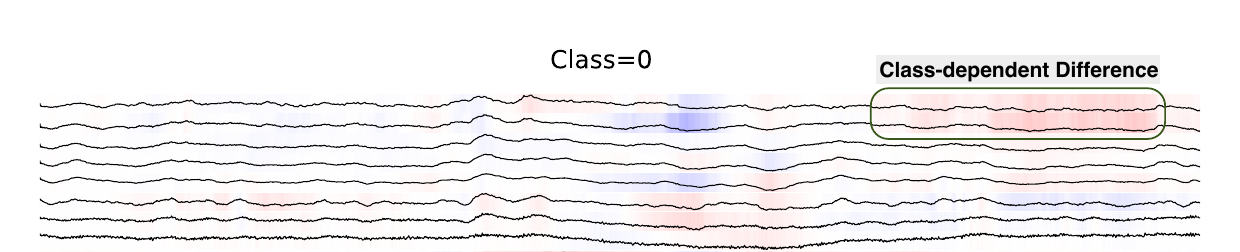}
        \vspace{-0.1em}        
        \raisebox{0pt}{ $\diagup\!\!\diagup$}
        \includegraphics[width=\linewidth, trim=0 3cm 0 22.57cm, clip]{windowed_heatmap_class_0.pdf}
    \end{subfigure}
    \hfill
    \begin{subfigure}{0.49\textwidth}
        \centering
        \includegraphics[width=\linewidth, trim=0 22.45cm 0 3cm, clip]{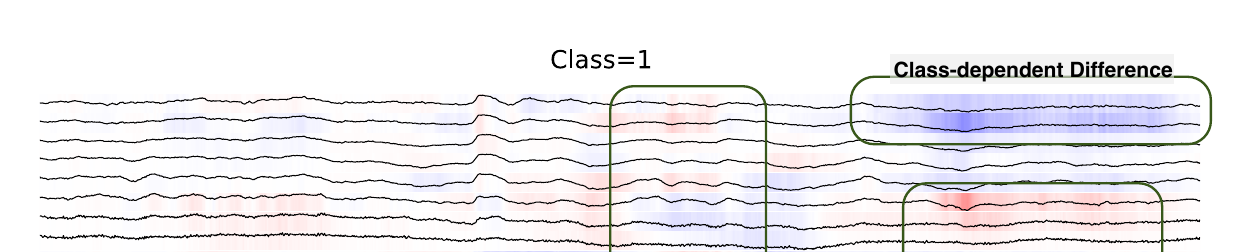}
        \vspace{-0.1em}        
        \raisebox{0pt}{ $\diagup\!\!\diagup$}
        \includegraphics[width=\linewidth, trim=0 3cm 0 22.57cm, clip]{windowed_heatmap_class_1.pdf}
    \end{subfigure}
    \caption{Averaged signals and relevance heatmaps for class 0 (low working-memory load) and class 1 (high working-memory load). Boxes highlight localized relevance, distributed multichannel relevance, stronger late-window relevance, and class-dependent attribution differences. Each time series stands for a sensor. Time range is 3-second with 3000 points. Red and blue indicate positive and negative relevance for the explained class score, respectively. 
    The full figure is in Figure~\ref{fig:seeg_visualization_full}.}
    \label{fig:seeg_visualization}
\end{figure}

Applying DFT-LRP \citep{DBLP:journals/pr/VielhabenLMS24}, we propagate relevance to the frequency space and visualize the relevance in Figure~\ref{fig:frequency_space_heatmap}. The results show that frequency-domain relevance is mainly concentrated in the lower-frequency range, particularly the delta band (0.5 - 4 Hz). Besides, the magnitude of delta-band relevance differs between the two working-memory load classes. These observations suggest the classifier mainly exploits low-frequency activity, with such activity contributing differently to the decision of the model under two working-memory load conditions.

\begin{figure}
    \centering
    \includegraphics[width=\linewidth]{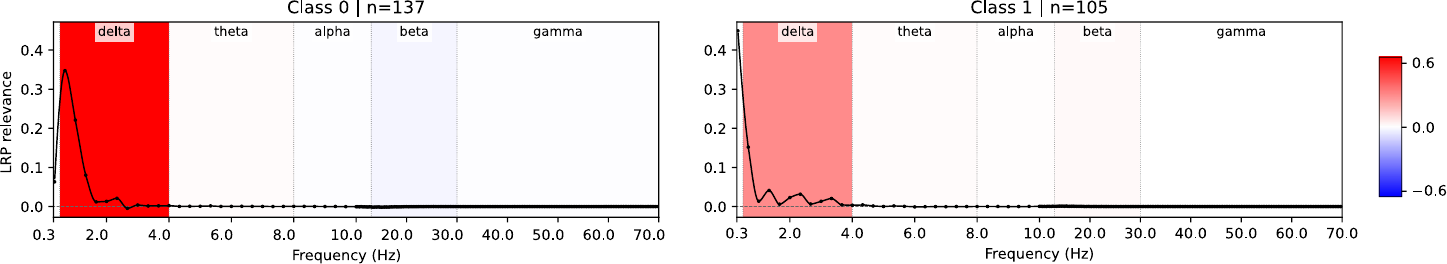}
    \caption{Relevance of sEEG prediction by class in frequency space. The relevance score is averaged on all trials of subject 2 in the dataset. We grouped the frequency-domain relevance into conventional electrophysiological frequency bands: delta (0.5–4 Hz), theta (4–8 Hz), alpha (8–13 Hz), beta (13–30 Hz), and gamma (30–70 Hz).}
    \label{fig:frequency_space_heatmap}
    \vspace{-0.8em}
\end{figure}

\noindent\textbf{Quantitative evaluation.}
We evaluate sensor-level attribution fidelity using all trials from all subjects in the dataset.
We follow the setting in \citep{li2026heegnet} and use subject-specific HEEGNet0 checkpoints from stratified five-fold within-subject cross-validation.
For an input $x\in\mathbb{R}^{C\times T}$, we fix the target class to the model's prediction and compute attributions for the target logit. The relevance of sensor $c$ is $R_c=\sum_{t=1}^{T}R_{c,t}$.
Sensors are ranked in descending order of $R_c$.
We compare LRP-radial-all, LRP-half, Gradient$\times$Input, Integrated Gradients (IG), and random rankings.
Both LRP variants use the $\epsilon$-rule with $\epsilon=10^{-9}$ for layers apart from hyperbolic layers.
IG uses a zero baseline and  integration over $128$ steps.
The random baseline averages five random rankings.

We replace perturbed sensors with the timepoint-wise mean across all sensors in the original trial.
We perturb the input by sparsity from $0\%$ to $100\%$ in $5\%$ increments, and calculate the fidelity as defined in  Eq.\ref{eq:sparsity},\ref{eq:fidplus}, as well as the area under the sparsity-fidelity curves.
We average trial-level AUCs across the held-out folds within each subject and then average the nine subject means, and the results are summarized in Table~\ref{tab:mnist_seeg_cifar_fidelity}. LRP-radial-all achieves the best average AUC of Fidelity$^+$ and comparable AUC of Fidelity$^-$ to the best baseline.

\subsection{CIFAR-10}
\label{sec:cifar}

We evaluate our relevance propagation method on CIFAR-10 using a lightweight classifier with six Lorentz convolutional layers inspired by \citet{DBLP:conf/iclr/BdeirSL24}. The model architecture and training details are in Appendix~\ref{app:model_cifar}. We assess the resulting explanations through qualitative heatmaps and quantitative fidelity and sparsity metrics. 

\noindent\textbf{Qualitative results.}
Figure~\ref{fig:cifar_qualitative} compares explanations for a CIFAR-10 image of ship. With the same $z^{\mathcal B}$-$\gamma$-$\epsilon$ configuration \citep{DBLP:journals/pr/MontavonLBSM17, DBLP:series/lncs/11700} for layers apart from hyperbolic layers, LRP-radial-all produces broader, spatially coherent attribution patterns that overlap visible object regions, whereas LRP-half concentrates attribution into smaller regions with positive and negative contributions. Gradient$\times$Input and Integrated Gradients exhibit much noisier patterns. These examples illustrate how the radial propagation rule affects attribution structure.

\begin{figure}
    \centering
    \includegraphics[width=0.8\linewidth]{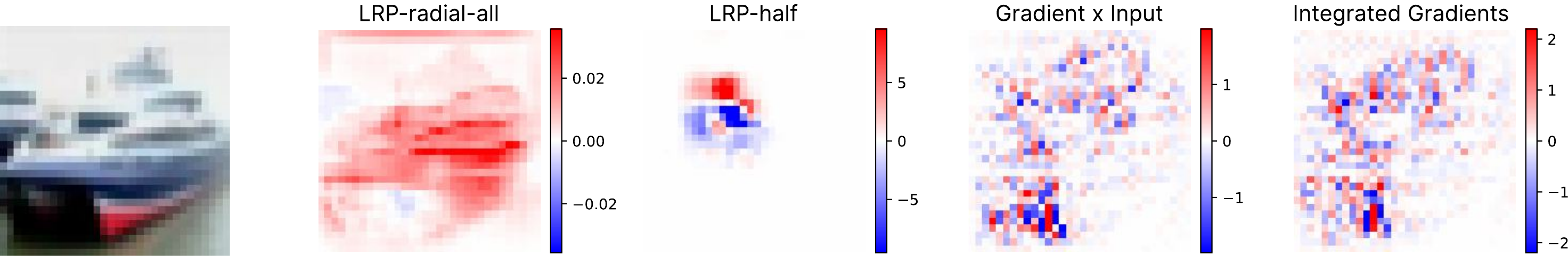}
    \caption{Qualitative attribution comparison on a correctly classified CIFAR-10 image of ship. Columns show the input, heatmaps for LRP-radial-all, LRP-half, Gradient$\times$Input, and Integrated Gradients. RGB-channel attributions are summed, with red and blue indicating positive and negative values, respectively. More examples in Figure~\ref{fig:cifar_qualitative_all}.}
    \label{fig:cifar_qualitative}
\end{figure}

\noindent\textbf{Quantitative evaluation.}
We evaluate attribution fidelity on 512 randomly sampled CIFAR-10 test images using pixel-wise mean-color replacement and predicted-class probabilities. As shown in Table~\ref{tab:mnist_seeg_cifar_fidelity}, LRP-radial-all achieves the lowest Fidelity$^{-}$ AUC, outperforming other methods, indicating stronger prediction preservation when retaining highly ranked pixels. However, its Fidelity$^{+}$ AUC falls even below Random. These results indicate stronger retention sufficiency but weaker deletion sensitivity under this perturbation protocol.

\subsection{Runtime evaluation}
\label{sec:runtime}

We compare attribution runtimes for MNIST on an AMD EPYC 7453, sEEG and CIFAR-10 on a NVIDIA A100.
As in Table~\ref{tab:runtime}, LRP-radial-all incurs substantially lower computational cost than Integrated Gradients, and comparable with Gradient$\times$Input that only includes one forward and backward computation of the model.

\section{Conclusion}
\label{sec:conclusion}

We developed a relevance propagation framework for origin-centered hyperbolic operations, guided by Geometric Representation Invariance and zero-curvature consistency. Our analysis shows that relevance conservation alone does not ensure either property, as the specified LRP-half baseline dependent on equivalent radial factorizations fails to approach identity propagation as curvature vanishes. 
By treating geometric scaling as modulation, LRP-radial-all conserves relevance locally, preserves the incoming allocation across radial factorizations, and satisfies zero-curvature consistency. 
These properties guarantee GRI for the specified Poincar\'e--Lorentz logarithmic-map construction. 
Experiments on MNIST, sEEG, and CIFAR-10 demonstrate competitive attribution fidelity with runtime comparable to Gradient$\times$Input and substantially lower than Integrated Gradients, while revealing a trade-off between retention and deletion performance on CIFAR-10. 

\noindent\textbf{Limitations and future work.}
Our theoretical guarantees apply to the specified module realizations and propagation conventions. 
The controlled exponential map test additionally requires a treatment of time-coordinate relevance, and invariance under alternative treatments remains an open question.
Future work includes extending the framework to non-origin base points and structured explanations of message-passing paths and higher-order interactions as in \citet{DBLP:journals/corr/abs-2606-00557}.

\subsubsection*{Acknowledgments}
This work was funded by the German Ministry for Education and Research as BIFOLD - Berlin Institute for the Foundations of Learning and Data (ref. BIFOLD25B). 
Thomas Schnake is a postdoctoral fellow at the University of Toronto in the Eric and Wendy Schmidt AI in Science Postdoctoral Fellowship Program, a program of Schmidt Sciences.

\bibliography{iclr2027_conference}

@article{peng2021hyperbolic,
  title={Hyperbolic deep neural networks: A survey},
  author={Peng, Wei and Varanka, Tuomas and Mostafa, Abdelrahman and Shi, Henglin and Zhao, Guoying},
  journal={IEEE Transactions on pattern analysis and machine intelligence},
  volume={44},
  number={12},
  pages={10023--10044},
  year={2021},
  publisher={IEEE},
doi={10.1109/TPAMI.2021.3136921}
}

@article{krioukov2010hyperbolic,
  title={Hyperbolic geometry of complex networks},
  author={Krioukov, Dmitri and Papadopoulos, Fragkiskos and Kitsak, Maksim and Vahdat, Amin and Bogun{\'a}, Mari{\'a}n},
  journal={Physical Review E—Statistical, Nonlinear, and Soft Matter Physics},
  volume={82},
  number={3},
  pages={036106},
  year={2010},
  publisher={APS},
doi={10.1103/PhysRevE.82.036106}
}

@article{DBLP:journals/pami/DombrowskiGMK24,
  author       = {Ann{-}Kathrin Dombrowski and
                  Jan E. Gerken and
                  Klaus{-}Robert M{\"{u}}ller and
                  Pan Kessel},
  title        = {Diffeomorphic Counterfactuals With Generative Models},
  journal      = {{IEEE} Trans. Pattern Anal. Mach. Intell.},
  volume       = {46},
  number       = {5},
  pages        = {3257--3274},
  year         = {2024}
}

@article{mettes2024hyperbolic,
  title={Hyperbolic deep learning in computer vision: A survey},
  author={Mettes, Pascal and Ghadimi Atigh, Mina and Keller-Ressel, Martin and Gu, Jeffrey and Yeung, Serena},
  journal={International Journal of Computer Vision},
  volume={132},
  number={9},
  pages={3484--3508},
  year={2024},
  publisher={Springer},
doi={10.1007/s11263-024-02043-5}
}

@inproceedings{
li2026heegnet,
title={{HEEGN}et: Hyperbolic Embeddings for {EEG}},
author={Shanglin Li and Chu Shiwen and Okan Ko{\c{c}} and Yi Ding and Qibin Zhao and Motoaki Kawanabe and Ziheng Chen},
booktitle={The Fourteenth International Conference on Learning Representations},
year={2026}
}

@article{zhou2026eeg,
  title={EEG-Based Multimodal Learning via Hyperbolic Mixture-of-Curvature Experts},
  author={Zhou, Runhe and Li, Shanglin and Huang, Guanxiang and Zhou, Xinliang and Zhao, Qibin and Kawanabe, Motoaki and Ding, Yi and Guan, Cuntai},
  journal={International Conference on Machine Learning},
  year={2026}
}

@inproceedings{desai2023hyperbolic,
  title={Hyperbolic image-text representations},
  author={Desai, Karan and Nickel, Maximilian and Rajpurohit, Tanmay and Johnson, Justin and Vedantam, Shanmukha Ramakrishna},
  booktitle={International Conference on Machine Learning},
  pages={7694--7731},
  year={2023},
  organization={PMLR}
}

@article{he2026helm,
  title={Helm: Hyperbolic large language models via mixture-of-curvature experts},
  author={He, Neil and Anand, Rishabh and Madhu, Hiren and Maatouk, Ali and Krishnaswamy, Smita and Tassiulas, Leandros and Yang, Menglin and Ying, Rex},
  journal={Advances in Neural Information Processing Systems},
  volume={38},
  pages={142604--142635},
  year={2026}
}

@inproceedings{khrulkov2020hyperbolic,
  title={Hyperbolic image embeddings},
  author={Khrulkov, Valentin and Mirvakhabova, Leyla and Ustinova, Evgeniya and Oseledets, Ivan and Lempitsky, Victor},
  booktitle={2020 IEEE/CVF conference on computer vision and pattern recognition (CVPR)},
  pages={6417--6427},
  year={2020},
  organization={IEEE}
}

@inproceedings{tifrea2018poincar,
author       = {Alexandru Tifrea and
                  Gary B{\'{e}}cigneul and
                  Octavian{-}Eugen Ganea},
  title        = {Poincar{\'e} Glove: Hyperbolic Word Embeddings},
  booktitle    = {{ICLR} (Poster)},
  publisher    = {OpenReview.net},
  year         = {2019}
}

@inproceedings{he2025hyperbolic,
  author       = {Neil He and
                  Hiren Madhu and
                  Ngoc Bui and
                  Menglin Yang and
                  Rex Ying},
  title        = {Hyperbolic Deep Learning for Foundation Models: {A} Survey},
  booktitle    = {{KDD} {(2)}},
  pages        = {6021--6031},
  publisher    = {{ACM}},
  year         = {2025}
}

@inproceedings{shimizu2020hyperbolic,
 author       = {Ryohei Shimizu and
                  Yusuke Mukuta and
                  Tatsuya Harada},
  title        = {Hyperbolic Neural Networks++},
  booktitle    = {{ICLR}},
  publisher    = {OpenReview.net},
  year         = {2021}
}

@inproceedings{zaher2024manifold,
  author       = {Eslam Zaher and
                  Maciej Trzaskowski and
                  Quan Nguyen and
                  Fred Roosta},
  title        = {Manifold Integrated Gradients: Riemannian Geometry for Feature Attribution},
  booktitle    = {{ICML}},
  series       = {Proceedings of Machine Learning Research},
  volume       = {235},
  pages        = {58090--58104},
  publisher    = {{PMLR} / OpenReview.net},
  year         = {2024}
}

@article{cannon1997hyperbolic,
  title={Hyperbolic geometry},
  author={Cannon, James W and Floyd, William J and Kenyon, Richard and Parry, Walter R},
  journal={Flavors of geometry},
  volume={31},
  pages={59-115},
  year={1997},
  publisher={Citeseer}
}

@article{Gunning2019,
  title = {XAI—Explainable artificial intelligence},
  volume = {4},
  number = {37},
  journal = {Science Robotics},
  publisher = {American Association for the Advancement of Science (AAAS)},
  author = {Gunning,  David and Stefik,  Mark and Choi,  Jaesik and Miller,  Timothy and Stumpf,  Simone and Yang,  Guang-Zhong},
  year = {2019},
}

@Article{DBLP:journals/pieee/SamekMLAM21,
  author    = {Wojciech Samek and Gr{\'{e}}goire Montavon and Sebastian Lapuschkin and Christopher J. Anders and Klaus-Robert M{\"{u}}ller},
  journal   = {Proc. {IEEE}},
  title     = {Explaining Deep Neural Networks and Beyond: {A} Review of Methods and Applications},
  year      = {2021},
  number    = {3},
  pages     = {247--278},
  volume    = {109},
  bibsource = {dblp computer science bibliography, https://dblp.org},
  biburl    = {https://dblp.org/rec/journals/pieee/SamekMLAM21.bib},
  __doi       = {10.1109/JPROC.2021.3060483},
  timestamp = {Sun, 06 Oct 2024 21:36:51 +0200},
}

@InProceedings{DBLP:conf/icml/HolzingerSMBS20,
author       = {Andreas Holzinger and
                  Anna Saranti and
                  Christoph Molnar and
                  Przemyslaw Biecek and
                  Wojciech Samek},
  title        = {Explainable {AI} Methods - {A} Brief Overview},
  booktitle    = {xxAI@ICML},
  __series       = {Lecture Notes in Computer Science},
  pages        = {13--38},
  publisher    = {Springer},
  year         = {2020}
}

@Book{DBLP:series/lncs/11700,
  editor       = {Wojciech Samek and
                  Gr{\'{e}}goire Montavon and
                  Andrea Vedaldi and
                  Lars Kai Hansen and
                  Klaus-Robert M{\"{u}}ller},
  title        = {Explainable {AI:} Interpreting, Explaining and Visualizing Deep Learning},
  series       = {Lecture Notes in Computer Science},
  volume       = {11700},
  publisher    = {Springer},
  year         = {2019}
}

@article{bach2015pixel,
  title={On pixel-wise explanations for non-linear classifier decisions by layer-wise relevance propagation},
  author={Bach, Sebastian and Binder, Alexander and Montavon, Gr{\'e}goire and Klauschen, Frederick and M{\"u}ller, Klaus-Robert and Samek, Wojciech},
  journal={PloS one},
  volume={10},
  number={7},
  pages={e0130140},
  year={2015},
  publisher={Public Library of Science San Francisco, CA USA}
}

@inproceedings{DBLP:conf/icml/SundararajanTY17,
  author       = {Mukund Sundararajan and
                  Ankur Taly and
                  Qiqi Yan},
  title        = {Axiomatic Attribution for Deep Networks},
  booktitle    = {{ICML}},
  series       = {Proceedings of Machine Learning Research},
  volume       = {70},
  pages        = {3319--3328},
  publisher    = {{PMLR}},
  year         = {2017}
}

@Article{DBLP:journals/pr/MontavonLBSM17,
  author    = {Gr{\'{e}}goire Montavon and Sebastian Lapuschkin and Alexander Binder and Wojciech Samek and Klaus-Robert M{\"{u}}ller},
  journal   = {Pattern Recognit.},
  title     = {Explaining nonlinear classification decisions with deep Taylor decomposition},
  year      = {2017},
  pages     = {211--222},
  volume    = {65},
  bibsource = {dblp computer science bibliography, https://dblp.org},
  biburl    = {https://dblp.org/rec/journals/pr/MontavonLBSM17.bib},
  __doi       = {10.1016/J.PATCOG.2016.11.008},
  timestamp = {Mon, 24 Feb 2020 08:30:09 +0100},
}

@inproceedings{arras2017recurrent,
    author       = {Leila Arras and
                  Gr{\'{e}}goire Montavon and
                  Klaus-Robert M{\"{u}}ller and
                  Wojciech Samek},
  title        = {Explaining Recurrent Neural Network Predictions in Sentiment Analysis},
  booktitle    = {WASSA@EMNLP},
  pages        = {159--168},
  publisher    = {Association for Computational Linguistics},
  year         = {2017}
}

@inproceedings{achtibat_attnlrp_2024,
  author       = {Reduan Achtibat and
                  Sayed Mohammad Vakilzadeh Hatefi and
                  Maximilian Dreyer and
                  Aakriti Jain and
                  Thomas Wiegand and
                  Sebastian Lapuschkin and
                  Wojciech Samek},
  title        = {AttnLRP: Attention-Aware Layer-Wise Relevance Propagation for Transformers},
  booktitle    = {{ICML}},
  __series       = {Proceedings of Machine Learning Research},
  pages        = {135--168},
  publisher    = {{PMLR} / OpenReview.net},
  year         = {2024}
}

@inproceedings{DBLP:conf/icml/AliSEMMW22,
  author       = {Ameen Ali and
                  Thomas Schnake and
                  Oliver Eberle and
                  Gr{\'{e}}goire Montavon and
                  Klaus-Robert M{\"{u}}ller and
                  Lior Wolf},
  __editor       = {Kamalika Chaudhuri and
                  Stefanie Jegelka and
                  Le Song and
                  Csaba Szepesv{\'{a}}ri and
                  Gang Niu and
                  Sivan Sabato},
  title        = {{XAI} for Transformers: Better Explanations through Conservative Propagation},
  booktitle    = {International Conference on Machine Learning, {ICML} 2022, 17-23 July
                  2022, Baltimore, Maryland, {USA}},
  __series       = {Proceedings of Machine Learning Research},
  volume       = {162},
  pages        = {435--451},
  publisher    = {{PMLR}},
  year         = {2022},
  __url          = {https://proceedings.mlr.press/v162/ali22a.html},
  timestamp    = {Tue, 12 Jul 2022 17:36:52 +0200},
  biburl       = {https://dblp.org/rec/conf/icml/AliSEMMW22.bib},
  bibsource    = {dblp computer science bibliography, https://dblp.org}
}

@inproceedings{DBLP:conf/nips/GaneaBH18,
  author       = {Octavian{-}Eugen Ganea and
                  Gary B{\'{e}}cigneul and
                  Thomas Hofmann},
  title        = {Hyperbolic Neural Networks},
  booktitle    = {NeurIPS},
  pages        = {5350--5360},
  year         = {2018}
}

@article{DBLP:journals/corr/abs-2606-00557,
  author       = {Ping Xiong and
                  Thomas Schnake and
                  Gr{\'{e}}goire Montavon and
                  Klaus{-}Robert M{\"{u}}ller and
                  Shinichi Nakajima},
  title        = {Normalized Relevance Measure as a Unifying Framework to Explain Neural
                  Network Latent Structures},
  journal      = {CoRR},
  volume       = {abs/2606.00557},
  year         = {2026}
}

@inproceedings{ancona2018towards,
  author       = {Marco Ancona and
                  Enea Ceolini and
                  Cengiz {\"{O}}ztireli and
                  Markus Gross},
  title        = {Towards better understanding of gradient-based attribution methods
                  for Deep Neural Networks},
  booktitle    = {{ICLR} (Poster)},
  publisher    = {OpenReview.net},
  year         = {2018}
}

@article{liu2019hyperbolic,
  title={Hyperbolic graph neural networks},
  author={Liu, Qi and Nickel, Maximilian and Kiela, Douwe},
  journal={Advances in neural information processing systems},
  volume={32},
  year={2019}
}

@InCollection{DBLP:series/lncs/ArrasAWMGMHS19,
  author    = {Leila Arras and Jose A. Arjona-Medina and Michael Widrich and Gr{\'{e}}goire Montavon and Michael Gillhofer and Klaus-Robert M{\"{u}}ller and Sepp Hochreiter and Wojciech Samek},
  booktitle = {Explainable {AI:} Interpreting, Explaining and Visualizing Deep Learning},
  publisher = {Springer},
  title     = {Explaining and Interpreting {LSTM}s},
  year      = {2019},
  __editor    = {Wojciech Samek and Gr{\'{e}}goire Montavon and Andrea Vedaldi and Lars Kai Hansen and Klaus-Robert M{\"{u}}ller},
  pages     = {211--238},
  __series    = {Lecture Notes in Computer Science},
  volume    = {11700},
  bibsource = {dblp computer science bibliography, https://dblp.org},
  biburl    = {https://dblp.org/rec/__series/lncs/ArrasAWMGMHS19.bib},
  __doi       = {10.1007/978-3-030-28954-6_11},
  timestamp = {Sat, 30 Sep 2023 10:30:32 +0200},
}

@article{boran2020dataset,
  title={Dataset of human medial temporal lobe neurons, scalp and intracranial EEG during a verbal working memory task},
  author={Boran, Ece and Fedele, Tommaso and Steiner, Adrian and Hilfiker, Peter and Stieglitz, Lennart and Grunwald, Thomas and Sarnthein, Johannes},
  journal={Scientific data},
  volume={7},
  number={1},
  pages={30},
  year={2020},
  publisher={Nature Publishing Group UK London},
doi={10.1038/s41597-020-0364-3}
}

@article{guillemaud2025hyperbolic,
  title={Hyperbolic embedding of brain networks as a tool for epileptic seizures forecasting},
  author={Guillemaud, Martin and Cousyn, Louis and Navarro, Vincent and Chavez, Mario},
  journal={Physical Review Research},
  volume={7},
  number={2},
  pages={023182},
  year={2025},
  publisher={APS}
}

@inproceedings{DBLP:conf/iclr/BdeirSL24,
  author       = {Ahmad Bdeir and
                  Kristian Schwethelm and
                  Niels Landwehr},
  title        = {Fully Hyperbolic Convolutional Neural Networks for Computer Vision},
  booktitle    = {{ICLR}},
  publisher    = {OpenReview.net},
  year         = {2024}
}

@article{DBLP:journals/pr/VielhabenLMS24,
  author       = {Johanna Vielhaben and
                  Sebastian Lapuschkin and
                  Gr{\'{e}}goire Montavon and
                  Wojciech Samek},
  title        = {Explainable {AI} for time series via Virtual Inspection Layers},
  journal      = {Pattern Recognit.},
  volume       = {150},
  pages        = {110309},
  year         = {2024}
}

@inproceedings{DBLP:conf/icml/NickelK18,
  author       = {Maximilian Nickel and
                  Douwe Kiela},
  title        = {Learning Continuous Hierarchies in the Lorentz Model of Hyperbolic
                  Geometry},
  booktitle    = {{ICML}},
  series       = {Proceedings of Machine Learning Research},
  volume       = {80},
  pages        = {3776--3785},
  publisher    = {{PMLR}},
  year         = {2018}
}

@Article{Yang2025,
author={Yang, Jiayi
and Cao, Dan
and Guo, Chunyan
and Stieglitz, Lennart
and Ledergerber, Debora
and Sarnthein, Johannes
and Li, Jin},
title={Enhanced role of the entorhinal cortex in adapting to increased working memory load},
journal={Nature Communications},
year={2025},
month={Jul},
day={01},
volume={16},
number={1},
pages={5798},
issn={2041-1723},
}
\bibliographystyle{iclr2027_conference}

\clearpage
\appendix
\section{Poincar\'e and Lorentz Representations}
\label{app:lorentz}
For curvature $-c$ with $c>0$, the Poincar\'e ball and Lorentz hyperboloid are
\begin{align}
\mathbb D_c^d&=\{\bm x_P\in\mathbb R^d:c\|\bm x_P\|_2^2<1\},\\
\mathbb L_c^d&=\{(x_{L,0},\bar{\bm x}_L)\in\mathbb R^{d+1}:-x_{L,0}^2+\|\bar{\bm x}_L\|_2^2=-1/c,\ x_{L,0}>0\}.
\end{align}
Their origins are $\bm o_P=\bm0$ and $\bm o_L=(c^{-1/2},\bm0)$. 
The standard Poincar\'e-Lorentz isometry and its inverse \citep{DBLP:conf/icml/NickelK18}, rescaled here to curvature $-c$, are 
\begin{align}
\phi(\bm x_P)&=\left(\frac{1+c\|\bm x_P\|_2^2}{\sqrt c(1-c\|\bm x_P\|_2^2)},\frac{2\bm x_P}{1-c\|\bm x_P\|_2^2}\right), \label{eq:poincare-lorentz}\\
\phi^{-1}(\bm x_L)&=\frac{\bar{\bm x}_L}{\sqrt c\,x_{L,0}+1}. \label{eq:inverse-phi}
\end{align}
Since $x_{L,0}=\sqrt{c^{-1}+\|\bar{\bm x}_L\|_2^2}$, both spatial conversions are radial:
\begin{equation}
\bar{\bm x}_L=\gamma(\bm x_P)\bm x_P,\qquad \gamma(\bm x_P)=\frac{2}{1-c\|\bm x_P\|_2^2},\qquad
\bm x_P=\frac{\bar{\bm x}_L}{\sqrt{1+c\|\bar{\bm x}_L\|_2^2}+1}.
\label{eq:lorentz-spatial}
\end{equation}

\section{Multiplicative Attribution and Grouping Invariance}
\label{app:multiplicative-classification}
Origin-centered hyperbolic maps and spatial coordinate conversions admit a scalar--signal structure, as in Eq.~\ref{eq:aligned-log}. This raises an attribution question: should relevance be shared between the geometric factor and the signal, or assigned entirely to the signal? Conservation alone does not determine this choice. To clarify the role of multiplicative structure, we first study attribution to scalar factors with equal explanatory status. This analysis complements the signal-only formulation used by LRP-radial-all.

\subsection{Local rules and grouping invariance}
For a scalar product $y=ab$ with $a,b>1$, consider a local rule that is linear in incoming relevance and whose coefficients depend only on the current factor values. Every conservative rule in this class has the form
\begin{equation}
R_a=w(a,b)R_y, \qquad R_b=[1-w(a,b)]R_y.
\label{eq:local-multiplicative-rule}
\end{equation}
The same coefficient function $w$ is used at every multiplication. Symmetric treatment of the factors additionally requires $w(a,b)=1-w(b,a)$.

\begin{definition}[Multiplicative grouping invariance]
\label{def:multiplicative-invariance}
A local relevance rule is multiplicatively grouping-invariant if, for every fixed ordered list of factors and incoming relevance, all binary parenthesizations of their product yield identical relevance at each original factor.
\end{definition}

This criterion depends on associative regrouping with the original explanatory factors held fixed. 
Equal splitting is symmetric and conservative but fails this criterion: for $y=abc$, the parenthesizations $(ab)c$ and $a(bc)$ assign relevance proportions $(1/4,1/4,1/2)$ and $(1/2,1/4,1/4)$, respectively.

More generally, propagation through $(ab)c$ assigns the coefficients
\begin{equation}
\bigl(w(ab,c)w(a,b),\; w(ab,c)[1-w(a,b)],\; 1-w(ab,c)\bigr)
\end{equation}
to $(R_a,R_b,R_c)$ relative to $R_y$, whereas propagation through $a(bc)$ assigns
\begin{equation}
\bigl(w(a,bc),\; [1-w(a,bc)]w(b,c),\; [1-w(a,bc)][1-w(b,c)]\bigr).
\end{equation}
Grouping invariance therefore requires, for every $a,b,c>1$,
\begin{align}
w(ab,c)w(a,b)&=w(a,bc), \label{eq:grouping-first}\\
w(ab,c)[1-w(a,b)]&=[1-w(a,bc)]w(b,c), \label{eq:grouping-middle}\\
1-w(ab,c)&=[1-w(a,bc)][1-w(b,c)]. \label{eq:grouping-last}
\end{align}
These identities are also sufficient: any two binary parenthesizations are connected by elementary reassociations, each of which preserves the relevance entering its three sub-expressions.

\subsection{Characterization of symmetric factor attribution}
\begin{theorem}[Continuous symmetric grouping-invariant attribution]
\label{thm:symmetric-multiplication}
Let $w:(1,\infty)^2\to\mathbb R$ be continuous. The conservative local rule in Eq.~\ref{eq:local-multiplicative-rule} is symmetric and multiplicatively grouping-invariant if and only if
\begin{equation}
R_a=\frac{\log a}{\log a+\log b}R_y, \qquad R_b=\frac{\log b}{\log a+\log b}R_y.
\label{eq:log-proportional-rule}
\end{equation}
\end{theorem}

\begin{proof}
Take incoming relevance $R_y=1$ and consider a product of $N$ identical factors $t>1$. Any adjacent pair can be made siblings by changing the parenthesization. Since symmetry gives $w(t,t)=1/2$, the two factors receive equal relevance in that parenthesization. Grouping invariance transfers this equality to every parenthesization. All adjacent factors therefore receive equal relevance, and conservation implies an allocation of $1/N$ to each factor.

Now consider $m+n$ identical factors, where $m,n$ are positive integers, and choose a tree whose root separates the first $m$ factors from the remaining $n$. Because the rule depends only on the two current factor values, the root assigns relevance $w(t^m,t^n)$ to the first subtree. Conservation within this subtree gives
\begin{equation}
w(t^m,t^n)=\frac{m}{m+n}.
\end{equation}
Consequently, whenever $\log a/\log b=m/n$, choosing $t=a^{1/m}=b^{1/n}$ yields
\begin{equation}
w(a,b)=\frac{\log a}{\log a+\log b}.
\end{equation}
This argument uses grouping invariance on the fixed list of identical factors and locality at the root and does not require a separate assumption of invariance under factor expansion.

For arbitrary $a,b>1$, choose positive rational numbers $q_r\to\log a/\log b$. Since $b^{q_r}\to a$, continuity extends the identity $w(b^{q_r},b)=q_r/(1+q_r)$ to $w(a,b)=\log a/(\log a+\log b)$.

Conversely, the logarithmic rule is continuous, symmetric, and conservative on the stated domain. For any product $y=\prod_{i=1}^{N}a_i$, propagation through an arbitrary binary tree yields
\begin{equation}
R_{a_i}=\frac{\log a_i}{\sum_{j=1}^{N}\log a_j}R_y,
\end{equation}
because intermediate logarithmic sums cancel along each root-to-leaf path. The allocation is therefore independent of parenthesization.
\end{proof}

\subsection{Connection to radial attribution}
This characterization applies to continuous, symmetric, value-based rules on $(1,\infty)^2$. The logarithmic formula does not directly cover general neural activations, since it is undefined for nonpositive factors and singular when positive factors have product one. Our geometric modules instead distinguish a vector signal from its scalar modulation, motivating the asymmetric LRP-radial-all convention. Its invariance under splitting or merging radial modules differs from regrouping a fixed list of explanatory factors. Appendix~\ref{app:radial-characterization} analyzes this setting within the fixed-proportion radial family.

\section{Characterization of Fixed-Proportion Radial Rules}
\label{app:radial-characterization}

\subsection{Rule family and repeated propagation}

Consider nonzero radial modules $\bm y=\alpha(\|\bm x\|_2)\bm x$ in dimension $d\geq2$, with nonzero intermediate signals. We restrict attention to the rule family in Eq.~\ref{eq:eta-radial-family}, where $\eta\in[0,1]$ is a fixed constant shared across modules. The signal branch receives proportion $\eta$ of each output relevance, while the scalar branch receives the remaining total relevance and redistributes it according to squared input coordinates.

For compactness, define the matrix
\begin{equation}
P_{\bm x}=\frac{\bm x\odot\bm x}{\|\bm x\|_2^2}\one^\top, \qquad T_{\eta,\bm x}=\eta I_d+(1-\eta)P_{\bm x},
\end{equation}
where $\odot$ denotes element-wise multiplication. The rule is $\bm R_{\bm x}=T_{\eta,\bm x}\bm R_{\bm y}$. Since the normalized squared coordinates sum to one,
\begin{equation}
P_{\bm x}^2=P_{\bm x}, \qquad \one^\top P_{\bm x}=\one^\top.
\end{equation}
Furthermore, $P_{a\bm x}=P_{\bm x}$ for every nonzero scalar $a$. Thus every module along a nonzero radial chain has the same squared-coordinate redistribution matrix.

\begin{proposition}[Factorization invariance within the fixed-proportion family]
\label{prop:eta-factorization}
For $d\geq2$, a rule in Eq.~\ref{eq:eta-radial-family} is invariant, for every nonzero input and incoming relevance, under replacement of one radial module by an equivalent chain of nonzero radial modules if and only if $\eta\in\{0,1\}$.
\end{proposition}

\begin{proof}
Since $P_{\bm x}$ is idempotent,
\begin{equation}
T_{\eta,\bm x}^m=\eta^m I_d+(1-\eta^m)P_{\bm x}
\end{equation}
for every $m\geq1$. In particular, invariance between one and two modules requires
\begin{equation}
T_{\eta,\bm x}^2-T_{\eta,\bm x}=(\eta^2-\eta)(I_d-P_{\bm x})=0.
\end{equation}
For $d\geq2$, $P_{\bm x}$ has rank one and cannot equal $I_d$. Therefore $\eta^2=\eta$, giving $\eta=0$ or $\eta=1$. Conversely, both choices yield idempotent backward operators, so repeated application does not change the propagated relevance.
\end{proof}

The case $\eta=1$ is radial-all. The case $\eta=0$ assigns all relevance to the scalar branch and then redistributes it according to squared input coordinates. Both rules conserve relevance and satisfy the stated radial factorization invariance, demonstrating that these two properties alone do not uniquely identify radial-all.

\subsection{Adding zero-curvature consistency}

\begin{corollary}[Unique joint consistency within the fixed-proportion family]
\label{cor:radial-all-characterization}
For $d\geq2$, LRP-radial-all is the unique member of the fixed-proportion family in Eq.~\ref{eq:eta-radial-family} satisfying both radial factorization invariance and zero-curvature consistency for the origin Poincar\'e maps in the fixed coordinates of Definition~\ref{def:zero-curvature}.
\end{corollary}

\begin{proof}
Proposition~\ref{prop:eta-factorization} restricts the candidates to $\eta=0$ and $\eta=1$. For fixed nonzero input and incoming relevance, $T_{\eta,\bm x}$ is independent of curvature. Zero-curvature (as in Definition~\ref{def:zero-curvature}) consistency therefore requires $T_{\eta,\bm x}\bm R=\bm R$ for every $\bm R$. The choice $\eta=1$ satisfies this requirement. The choice $\eta=0$ does not, since $P_{\bm x}\neq I_d$ for $d\geq2$. Hence only $\eta=1$ satisfies both criteria.
\end{proof}

\section{Origin Logarithmic Maps and the GRI Test}
\label{app:log}
\subsection{Poincar\'e branch}
Let $r_P=\|\bm x_P\|_2$. The origin logarithmic map has the radial form
\begin{equation}
\bm v=\log_{\bm o_P}^{P,c}(\bm x_P)=\alpha_P(\bm x_P)\bm x_P,\qquad \alpha_P(\bm x_P)=\frac{\operatorname{artanh}(\sqrt c\,r_P)}{\sqrt c\,r_P}.
\label{eq:poincare-log}
\end{equation}
The continuous value at the origin is $\alpha_P(\bm0)=1$.

\subsection{Lorentz branch and tangent alignment}
Let $\theta_L=\operatorname{arcosh}(\sqrt c\,x_{L,0})$. The  Lorentz logarithmic map is
\begin{equation}
\log_{\bm o_L}^{L,c}(\bm x_L)=\frac{\theta_L}{\sinh\theta_L}\left(\bm x_L-\sqrt c\,x_{L,0}\bm o_L\right)=\left(0,\frac{\theta_L}{\sinh\theta_L}\bar{\bm x}_L\right).
\label{eq:lorentz-log-ambient}
\end{equation}
To compare this output with the Poincar\'e branch, we express both in common tangent coordinates. The Jacobian of the isometry in Eq.~\ref{eq:poincare-lorentz} satisfies
\begin{equation}
\left.\frac{\partial\phi(\bm x_P)}{\partial\bm x_P}\right|_{\bm x_P=\bm0}\bm v
=\begin{pmatrix}\bm0^\top\\2I_d\end{pmatrix}\bm v
=(0,2\bm v).
\label{eq:tangent-alignment}
\end{equation}
Thus, the Lorentz spatial tangent component must be divided by two to recover the common coordinates:
\begin{equation}
\bm v=\frac12\left[\log_{\bm o_L}^{L,c}(\bm x_L)\right]_{\mathrm{sp}}=\alpha_L(\bm x_L)\bar{\bm x}_L,\qquad \alpha_L(\bm x_L)=\frac{\theta_L}{2\sinh\theta_L}.
\label{eq:lorentz-log}
\end{equation}
Here $[\cdot]_{\mathrm{sp}}$ extracts the spatial components. Let $r_L=\|\bar{\bm x}_L\|_2$, the hyperboloid constraint gives $\sqrt c\,x_{L,0}=\sqrt{1+cr_L^2}$ and $\sinh\theta_L=\sqrt c\,r_L$. Consequently,
\begin{equation}
\alpha_L(\bm x_L)=\frac{\operatorname{arsinh}(\sqrt c\,r_L)}{2\sqrt c\,r_L},
\end{equation}
so the aligned Lorentz logarithmic map is radial in the spatial coordinates, with $\alpha_L(\bm o_L)=1/2$.

\subsection{Forward equivalence}
For $\bm x_L=\phi(\bm x_P)$ and $t=\sqrt c\,r_P\in[0,1)$, the conversion formulas give
\begin{equation}
\sqrt c\,x_{L,0}=\frac{1+t^2}{1-t^2},\qquad \theta_L=2\operatorname{artanh}(t),\qquad \sinh\theta_L=\frac{2t}{1-t^2}.
\end{equation}
Together with $\bar{\bm x}_L=\gamma(\bm x_P)\bm x_P$,
\begin{equation}
\alpha_L(\phi(\bm x_P))\gamma(\bm x_P)=\frac{2\operatorname{artanh}(t)}{2[2t/(1-t^2)]}\frac{2}{1-t^2}=\frac{\operatorname{artanh}(t)}{t}=\alpha_P(\bm x_P).
\end{equation}
All expressions at $t=0$ are interpreted by continuity. The two paths therefore produce identical common tangent coordinates:
\begin{equation}
\bm v=\log_{\bm o_P}^{P,c}(\bm x_P)=\frac12\left[\log_{\bm o_L}^{L,c}(\phi(\bm x_P))\right]_{\mathrm{sp}}.
\label{eq:log-forward-equivalence}
\end{equation}

\subsection{Relevance equality in the controlled test}
Under the shared encoder, downstream computation, and relevance procedure specified in Section~\ref{sec:gri-test}, both paths receive the same $\bm R_{\bm v}$. LRP-radial-all propagates this relevance unchanged through each spatial radial module. The direct path contains the Poincar\'e logarithmic map, while the Lorentz path contains the aligned logarithmic map and the spatial coordinate conversion. Hence,
\begin{equation}
\bm R_{\bm x_P}^{P}=\bm R_{\bm v},\qquad \bm R_{\bm x_P}^{L}=\bm R_{\bar{\bm x}_L}=\bm R_{\bm v}.
\end{equation}
The Lorentz time coordinate is treated as a dependent geometric quantity within these modules. Identical relevance at $\bm x_P$, followed by the same deterministic propagation through the shared encoder, gives
\begin{equation}
\bm R_{\bm a}^{\mathrm{direct}}=\bm R_{\bm a}^{\mathrm{via}\,L}.
\end{equation}
LRP-radial-all therefore satisfies GRI for this specified logarithmic-map test.

\section{Exponential Maps and the Tests}
\label{app:exp}
\subsection{Poincar\'e exponential map}
Let $r_z=\|\bm z\|_2$. The origin exponential map has the radial form
\begin{equation}
\bm y_P=\exp_{\bm o_P}^{P,c}(\bm z)=\beta_P(\bm z)\bm z,\qquad \beta_P(\bm z)=\frac{\tanh(\sqrt c\,r_z)}{\sqrt c\,r_z}.
\label{eq:poincare-exp}
\end{equation}
The continuous value at the origin is $\beta_P(\bm0)=1$.

\subsection{Lorentz exponential map and alignment}
Under the tangent alignment in Eq.~\ref{eq:tangent-alignment}, the common coordinates $\bm z$ correspond to the ambient Lorentz tangent vector $(0,2\bm z)$. For an origin tangent vector $\bm u=(0,\bar{\bm u})$, with $\|\bm u\|_L=\|\bar{\bm u}\|_2$, the exponential map is
\begin{equation}
\exp_{\bm o_L}^{L,c}(\bm u)=\cosh(\sqrt c\,\|\bm u\|_L)\bm o_L+\frac{\sinh(\sqrt c\,\|\bm u\|_L)}{\sqrt c\,\|\bm u\|_L}\bm u.
\end{equation}
Substituting $\bm u=(0,2\bm z)$ gives
\begin{equation}
\bm y_L=\exp_{\bm o_L}^{L,c}((0,2\bm z))=\left(\frac{\cosh(2\sqrt c\,r_z)}{\sqrt c},\beta_L(\bm z)\bm z\right),\qquad \beta_L(\bm z)=\frac{\sinh(2\sqrt c\,r_z)}{\sqrt c\,r_z}.
\label{eq:lorentz-exp}
\end{equation}
Thus, the spatial output $\bar{\bm y}_L=\beta_L(\bm z)\bm z$ is radial, with $\beta_L(\bm0)=2$.

To verify forward equivalence, set $t=\sqrt c\,r_z$. Taking the norm of Eq.~\ref{eq:poincare-exp} gives $\|\bm y_P\|_2=\tanh(t)/\sqrt c$. Substitution into Eq.~\ref{eq:poincare-lorentz} then yields
\begin{align}
\phi(\bm y_P)
&=\left(\frac{1+\tanh^2t}{\sqrt c(1-\tanh^2t)},\frac{2\tanh t}{t(1-\tanh^2t)}\bm z\right)\notag\\
&=\left(\frac{\cosh(2t)}{\sqrt c},\frac{\sinh(2t)}{t}\bm z\right)
=\bm y_L,
\end{align}
where the expressions at $t=0$ are interpreted by continuity. The two exponential maps therefore produce corresponding geometric points when their tangent inputs are aligned.

\subsection{Relevance rules for the spatial radial modules}
Both $\bm y_P=\beta_P(\bm z)\bm z$ and $\bar{\bm y}_L=\beta_L(\bm z)\bm z$ have the scalar--signal structure of Eq.~\ref{eq:geometry-weight-rule}. We apply the same propagation rules as for the logarithmic maps, with $\bm z$ as the signal. The constant tangent alignment is included in $\beta_L$. LRP-radial-all and LRP-half are applicable here. 

\subsection{A controlled exponential-map test}
The exponential-map test compares two paths from the same tangent coordinates $\bm z$ to the same Lorentz point. The direct path applies the aligned Lorentz exponential map, whereas the indirect path applies the Poincar\'e exponential map followed by the coordinate conversion:
\begin{equation}
\bm y_L=\exp_{\bm o_L}^{L,c}((0,2\bm z))=\phi\!\left(\exp_{\bm o_P}^{P,c}(\bm z)\right).
\end{equation}
Both paths then use the same downstream computation $G(\bm y_L)$. With identical output relevance initialization and the same deterministic downstream propagation, they receive the same relevance $\bm R_{\bm y_L}=(R_{y_{L,0}},\bm R_{\bar{\bm y}_L})$.

\paragraph{Shared treatment of the time coordinate.}
In both realizations, we express the time coordinate through the same spatial constraint,
\begin{equation}
y_{L,0}=\sqrt{c^{-1}+\|\bar{\bm y}_L\|_2^2}.
\end{equation}
Any relevance assigned to this coordinate is propagated through the constraint using the same rule in both paths. For example, for $\bar{\bm y}_L\neq\bm0$, squared-coordinate redistribution gives
\begin{equation}
\widetilde R_{\bar y_{L,i}}=R_{\bar y_{L,i}}+\frac{\bar y_{L,i}^2}{\|\bar{\bm y}_L\|_2^2}R_{y_{L,0}}.
\label{eq:exp-time-redistribution}
\end{equation}
This convention conserves the combined spatial and time-coordinate relevance. At zero spatial input, a shared fallback must be specified. The equality argument below requires only that both paths use the same deterministic treatment of the time coordinate.

\paragraph{Relevance equality under LRP-radial-all.}
After the shared time-coordinate propagation, both paths receive the same effective spatial relevance $\widetilde{\bm R}_{\bar{\bm y}_L}$. The direct spatial map is $\bar{\bm y}_L=\beta_L(\bm z)\bm z$, while the indirect path contains two radial modules:
\begin{equation}
\bm y_P=\beta_P(\bm z)\bm z,\qquad
\bar{\bm y}_L=\gamma(\bm y_P)\bm y_P.
\end{equation}
Applying LRP-radial-all to each spatial radial module gives
\begin{equation}
\bm R_{\bm z}^{\mathrm{direct}}=\widetilde{\bm R}_{\bar{\bm y}_L},
\qquad
\bm R_{\bm z}^{\mathrm{via}\,P}=\bm R_{\bm y_P}=\widetilde{\bm R}_{\bar{\bm y}_L}.
\end{equation}
Thus, the two paths yield identical relevance at the common input $\bm z$. If $\bm z$ is produced by a shared encoder, identical deterministic propagation through that encoder also gives identical relevance at the original input.

\section{LRP-half: Conservation and Consistency Counterexamples}
\label{app:half}
This section analyzes the specified LRP-half baseline, which combines equal splitting at scalar--signal products with squared-norm redistribution of the scalar relevance. Although this rule conserves total relevance, it can violate GRI and zero-curvature consistency.

\subsection{Branch split, squared-norm redistribution, and conservation}
Consider a radial module $\bm y=\alpha(\|\bm x\|_2)\bm x$ with $\bm x\neq\bm0$. LRP-half assigns half of the incoming relevance to the signal branch and half to the scalar factor:
\begin{equation}
R_{x_i}^{\mathrm{sig}}=\frac12R_{y_i},\qquad R_\alpha=\frac12\sum_jR_{y_j}.
\end{equation}
Let $s=\|\bm x\|_2^2=\sum_jx_j^2$ and $r=\sqrt{s}$. We treat the scalar transformations $s\mapsto r\mapsto\alpha(r)$ as relevance-preserving, so $R_s=R_r=R_\alpha$. Contribution-proportional redistribution through the sum, followed by relevance-preserving propagation through each scalar square, gives
\begin{equation}
R_{x_i}^{\mathrm{rad}}=\frac{x_i^2}{\|\bm x\|_2^2}R_s.
\end{equation}
Combining the two branches yields
\begin{equation}
R_{x_i}=\frac12R_{y_i}+\frac12\frac{x_i^2}{\|\bm x\|_2^2}\sum_jR_{y_j}.
\label{eq:lrp_half_radial}
\end{equation}
Since the squared-coordinate weights sum to one, the rule conserves signed total relevance:
\begin{equation}
\sum_iR_{x_i}=\frac12\sum_iR_{y_i}+\frac12\sum_jR_{y_j}=\sum_jR_{y_j}.
\end{equation}
At $\bm x=\bm0$, the redistribution weights are undefined and require a separate convention, such as identity propagation. The counterexamples below use nonzero inputs and are independent of this choice. Numerical stabilization is discussed in Appendix~\ref{app:numerics}.

\begin{figure}[t]
\centering
\includegraphics[width=0.4\linewidth]{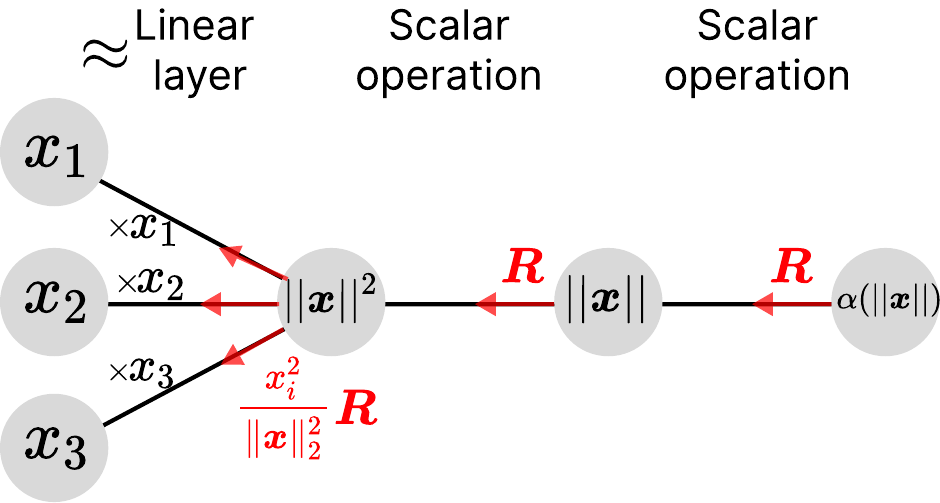}
\caption{Redistribution of scalar-branch relevance from $\alpha(\|\bm x\|_2)$ to $\bm x$ through the squared norm.}
\label{fig:alpha_x}
\end{figure}

\subsection{Factorization dependence and the GRI counterexample}
Consider the aligned logarithmic-map paths in Eq.~\ref{eq:aligned-log}, with $\bm x_P\neq\bm0$ and identical incoming relevance $\bm R_{\bm v}$. The direct Poincar\'e path contains one radial module and gives
\begin{equation}
R_{x_{P,i}}^{P}=\frac12R_{v_i}+\frac12\frac{x_{P,i}^2}{\|\bm x_P\|_2^2}\sum_jR_{v_j}.
\end{equation}
The Lorentz path contains two radial modules: the spatial conversion $\bar{\bm x}_L=\gamma(\bm x_P)\bm x_P$ and the aligned logarithmic map. Because the conversion scales every coordinate by the same nonzero scalar,
\begin{equation}
\frac{\bar{x}_{L,i}^2}{\|\bar{\bm x}_L\|_2^2}=\frac{x_{P,i}^2}{\|\bm x_P\|_2^2}.
\label{eq:half-radial-weights}
\end{equation}
Propagation through the aligned Lorentz logarithmic map first gives
\begin{equation}
R_{\bar{x}_{L,i}}=\frac12R_{v_i}+\frac12\frac{\bar{x}_{L,i}^2}{\|\bar{\bm x}_L\|_2^2}\sum_jR_{v_j},\qquad
\sum_jR_{\bar{x}_{L,j}}=\sum_jR_{v_j}.
\label{eq:lorentz_half}
\end{equation}
Applying the rule again through the spatial conversion yields
\begin{equation}
\begin{aligned}
R_{x_{P,i}}^{L}
&=\frac12R_{\bar{x}_{L,i}}+\frac12\frac{x_{P,i}^2}{\|\bm x_P\|_2^2}\sum_jR_{\bar{x}_{L,j}}\\
&=\frac12\left(\frac12R_{v_i}+\frac12\frac{x_{P,i}^2}{\|\bm x_P\|_2^2}\sum_jR_{v_j}\right)
+\frac12\frac{x_{P,i}^2}{\|\bm x_P\|_2^2}\sum_jR_{v_j}\\
&=\frac14R_{v_i}+\frac34\frac{x_{P,i}^2}{\|\bm x_P\|_2^2}\sum_jR_{v_j}.
\end{aligned}
\end{equation}
The constant tangent alignment is included in $\alpha_L$ and is not treated as an additional half-splitting node. Both paths conserve the same total relevance, but their coordinate-wise difference is
\begin{equation}
R_{x_{P,i}}^{L}-R_{x_{P,i}}^{P}=\frac14\left(\frac{x_{P,i}^2}{\|\bm x_P\|_2^2}\sum_jR_{v_j}-R_{v_i}\right).
\end{equation}
The allocations coincide if and only if $R_{v_i}=(x_{P,i}^2/\|\bm x_P\|_2^2)\sum_jR_{v_j}$ for every $i$.

\paragraph{An explicit counterexample at the original input.}
Choose the shared encoder to be the identity on the Poincar\'e ball, so that $\bm a=\bm x_P$. Let $c=1$, $\bm x_P^\ast=(1/4,1/4)^\top$, and $\bm v^\ast=\log_{\bm o_P}^{P,1}(\bm x_P^\ast)$. Use the shared linear head $G(\bm v)=v_1/v_1^\ast$, where $v_1^\ast$ is fixed. At this input, the output score is one, and LRP-$0$ through the head with unit relevance initialization gives $\bm R_{\bm v}=(1,0)^\top$. The two paths yield
\begin{equation}
\bm R_{\bm a}^{\mathrm{direct}}=(3/4,1/4)^\top,\qquad
\bm R_{\bm a}^{\mathrm{via}\,L}=(5/8,3/8)^\top.
\end{equation}
Both allocations sum to one but differ at the same original input. Since the aligned paths implement the same function throughout the common domain, this is a counterexample to GRI for the specified LRP-half baseline.

\subsection{Failure of zero-curvature consistency}
For fixed inputs, the origin Poincar\'e logarithmic and exponential maps satisfy
\begin{align}
\log_{\bm o_P}^{P,c}(\bm x)&=\left(1+\frac{c}{3}\|\bm x\|_2^2+O(c^2)\right)\bm x,\\
\exp_{\bm o_P}^{P,c}(\bm z)&=\left(1-\frac{c}{3}\|\bm z\|_2^2+O(c^2)\right)\bm z.
\end{align}
Both converge to the identity as $c\to0$. However, for fixed nonzero input and incoming relevance, the LRP-half allocation in Eq.~\ref{eq:lrp_half_radial} is independent of curvature. Consequently,
\begin{equation}
\lim_{c\to0}\left(R_{x_i}-R_{y_i}\right)=\frac12\left(\frac{x_i^2}{\|\bm x\|_2^2}\sum_jR_{y_j}-R_{y_i}\right),
\end{equation}
which is generally nonzero. Although the radial factor approaches one, the rule continues to redistribute half of the incoming relevance through the radial branch. The same argument applies to the exponential map with $\bm z$ as the input signal. Thus, the specified LRP-half baseline fails zero-curvature consistency for both Poincar\'e origin maps. 

\section{Bias, M\"obius Addition, and Matched Lorentz Realization}
\label{app:bias}

\subsection{Poincar\'e bias operation}

Let $\bm b$ be a fixed tangent-space parameter and $\bm b_P=\exp_{\bm o_P}^{P,c}(\bm b)$. For $\bm m_P=\bm W\otimes_c\bm x_P$, define the bias operation as
\begin{equation}
\bm y_P=T_{\bm b_P}^{P}(\bm m_P)=\bm m_P\oplus_c\bm b_P.
\end{equation}
M\"obius addition can also be represented in the scalar-vector product form
\begin{equation}
\bm m_P\oplus_c\bm b_P=\alpha^\oplus_c(\bm m_P,\bm b_P)\bm m_P+\beta^\oplus_c(\bm m_P,\bm b_P)\bm b_P,
\label{eq:mobius-bias}
\end{equation}
where
\begin{equation}
\begin{aligned}
\alpha^\oplus_c(\bm m_P,\bm b_P)
&=\frac{1+2c\langle\bm m_P,\bm b_P\rangle+c\|\bm b_P\|_2^2}{1+2c\langle\bm m_P,\bm b_P\rangle+c^2\|\bm m_P\|_2^2\|\bm b_P\|_2^2},\\
\beta^\oplus_c(\bm m_P,\bm b_P)
&=\frac{1-c\|\bm m_P\|_2^2}{1+2c\langle\bm m_P,\bm b_P\rangle+c^2\|\bm m_P\|_2^2\|\bm b_P\|_2^2}.
\end{aligned}
\label{eq:mobius-addition}
\end{equation}
Although $\bm b_P$ is fixed, both coefficients depend on the input $\bm m_P$. The bias contribution generally changes the output direction, so M\"obius bias addition is not an origin-centered radial scaling of the signal.

\subsection{Signal-only relevance convention}

We explain the prediction in terms of input features rather than fixed model parameters. In this spirit, we suggest one way to propagate relevance conservatively:
\begin{equation}
\bm R_{\bm b_P}=\bm0, \qquad \bm R_{\bm m_P}=\bm R_{\bm y_P}.
\label{eq:mobius-bias-lrp}
\end{equation}
Note it is a separate attribution convention and not a consequence of the radial factorization result.

\subsection{Matched Lorentz realization}

An exactly equivalent forward bias operation can be constructed through the Poincar\'e--Lorentz isometry $\phi$. Set $\bm m_L=\phi(\bm m_P)$ and $\bm b_L=\phi(\bm b_P)$, and define
\begin{align}
T_{\bm b_L}^{L}&=\phi\circ T_{\bm b_P}^{P}\circ\phi^{-1}, \label{eq:lorentz-bias-conjugation}\\
T_{\bm b_L}^{L}(\bm m_L)&=\phi\!\left(\phi^{-1}(\bm m_L)\oplus_c\bm b_P\right).
\end{align}
The resulting computation is
\begin{equation}
\bm m_L\xrightarrow{\phi^{-1}}\bm m_P\xrightarrow{\oplus_c\bm b_P}\bm y_P\xrightarrow{\phi}\bm y_L.
\end{equation}
By construction, $T_{\bm b_L}^{L}(\phi(\bm m_P))=\phi(T_{\bm b_P}^{P}(\bm m_P))$, so the two realizations implement the same operation in different geometric representations. 

\section{Tangent-space propagation and numerical stabilization}
\label{app:numerics}

For a bias-free tangent-space linear transformation $\bm z=\bm W\bm v$, the unstabilized LRP-$0$ rule is
\begin{equation}
R_{v_i}=\sum_j\frac{v_iW_{ji}}{z_j}R_{z_j}, \qquad z_j=\sum_kv_kW_{jk}.
\label{eq:linear-lrp}
\end{equation}
Assuming $z_j\neq0$ for every output receiving nonzero relevance, and omitting zero-relevance outputs, summation over the inputs gives exact conservation:
\begin{equation}
\sum_iR_{v_i}=\sum_j\frac{\sum_iv_iW_{ji}}{z_j}R_{z_j}=\sum_jR_{z_j}.
\end{equation}

Numerical stabilization generally introduces a conservation residual. For LRP-$\epsilon$ with $\epsilon>0$, replace the denominator by $\widetilde z_j=z_j+\epsilon\operatorname{sgn}_{+}(z_j)$, where $\operatorname{sgn}_{+}(z)=1$ for $z\geq0$ and $-1$ otherwise. The propagated total becomes
\begin{equation}
\sum_iR_{v_i}=\sum_j\frac{z_j}{z_j+\epsilon\operatorname{sgn}_{+}(z_j)}R_{z_j}.
\end{equation}
Consequently, the signed conservation residual is
\begin{equation}
\sum_jR_{z_j}-\sum_iR_{v_i}=\sum_j\frac{\epsilon}{|z_j|+\epsilon}R_{z_j}.
\end{equation}
Exact conservation is therefore not guaranteed unless the residual vanishes or is explicitly handled, for example as in NRM \citep{DBLP:journals/corr/abs-2606-00557}.

\section{Geometric Representation Invariance Experiment}
\label{sec:exp_geometry}
\paragraph{Setup.}
We consider curvature $-c$ with $c=1$ and choose the shared encoder to be the identity on the Poincar\'e ball, so that the original input is $\bm a=\bm x_P$. We evaluate at $\bm x_P=(0.2,0.3,-0.1)^\top$. Its Lorentz representation is
\begin{equation}
\bm x_L=\phi(\bm x_P)=\left(\frac{1+c\|\bm x_P\|_2^2}{\sqrt c(1-c\|\bm x_P\|_2^2)},\frac{2\bm x_P}{1-c\|\bm x_P\|_2^2}\right).
\end{equation}
The direct Poincar\'e path and the aligned Lorentz path reach the same tangent coordinates:
\begin{equation}
\bm v=\log_{\bm o_P}^{P,c}(\bm x_P)=\frac12\left[\log_{\bm o_L}^{L,c}\!\left(\phi(\bm x_P)\right)\right]_{\mathrm{sp}},
\qquad \bm o_P=\bm0,\quad \bm o_L=(c^{-1/2},\bm0).
\end{equation}
Here $[\cdot]_{\mathrm{sp}}$ extracts the spatial components, and the factor $1/2$ converts the ambient Lorentz tangent vector to the common coordinates, as specified in Eq.~\ref{eq:tangent-alignment}. Both paths use the shared scalar head $G(\bm v)=v_1$ and implement the same function throughout the Poincar\'e ball. Initializing output relevance with the target score and applying LRP-$0$ through this head gives $\bm R_{\bm v}=(v_1,0,0)^\top$.

\paragraph{Relevance propagation.}
For a radial module $\bm y=\alpha(\|\bm x\|_2)\bm x$, LRP-radial-all propagates $\bm R_{\bm x}=\bm R_{\bm y}$, whereas the specified LRP-half baseline uses
\begin{equation}
R_{x_i}=\frac12R_{y_i}+\frac12\frac{x_i^2}{\|\bm x\|_2^2}\sum_jR_{y_j},
\qquad \bm x\neq\bm0.
\label{eq:exp_half}
\end{equation}
Both rules conserve signed total relevance locally. The direct path contains one radial module, $\bm v=\alpha_P(\bm x_P)\bm x_P$. The Lorentz path contains two: the spatial conversion $\bar{\bm x}_L=\gamma(\bm x_P)\bm x_P$ and the aligned logarithmic map $\bm v=\alpha_L(\bm x_L)\bar{\bm x}_L$. We apply each rule consistently to these modules and compare the resulting relevance at the common original input $\bm a=\bm x_P$. The tangent alignment factor is included in $\alpha_L$, rather than treated as an additional multiplication node. The Lorentz time coordinate is treated as a dependent geometric quantity within the radial modules.

\begin{table}[t]
\centering
\small
\caption{Relevance at the common original input $\bm a=\bm x_P$ for the two equivalent logarithmic-map paths. Values are rounded to six decimal places.}
\label{tab:geometry_invariance}
\begin{tabular}{llrrr}
\toprule
Rule & Path & $R_{a_1}$ & $R_{a_2}$ & $R_{a_3}$ \\
\midrule
LRP-radial-all & Direct & 0.210205 & 0 & 0 \\
& Via Lorentz & 0.210205 & 0 & 0 \\
LRP-half & Direct & 0.135132 & 0.067566 & 0.007507 \\
& Via Lorentz & 0.097595 & 0.101349 & 0.011261 \\
\bottomrule
\end{tabular}
\end{table}

\paragraph{Results.}
Table~\ref{tab:geometry_invariance} shows identical input relevance for LRP-radial-all across the two paths. In contrast, LRP-half produces different relevances, although both allocations sum to the same target score $v_1$. This numerical example illustrates the counterexample derived in Appendix~\ref{app:half}: local conservation does not guarantee GRI, because equivalent radial factorizations can yield different relevance allocations at the same original input.

\section{Validate sEEG relevant window explanation by perturbation}
\label{app:seeg_validation}

To assess whether the relevant time windows in the LRP heatmaps identify input regions relevant to the network's predictions, we conduct a window perturbation test. We compute trial-wise relevance for the target class logit and averaged the maps within each class.
The average input time series (3 seconds long) is divided into non-overlapping 100-ms windows, ranked by their signed relevance summed across channels and time.
We perturb the top $k \in \{1,2,3,5\}$ windows in each original trial by replacing the signal within these windows, across all channels, with each channel's original temporal mean. Randomly selected windows serve as a control, averaged over 20 repetitions. The results are summarized in Figure~\ref{fig:window_perturbation}.
LRP-guided perturbation produces larger mean reductions in target probabilities than random perturbation for both classes. These results indicate that the temporal relevance estimated by LRP is aligned with the model's predictive behavior and can effectively identify time periods that are important for the classification decision. This further supports the practical utility of the relevance heatmaps.  
\begin{figure}
    \centering
    \includegraphics[width=0.6\linewidth]{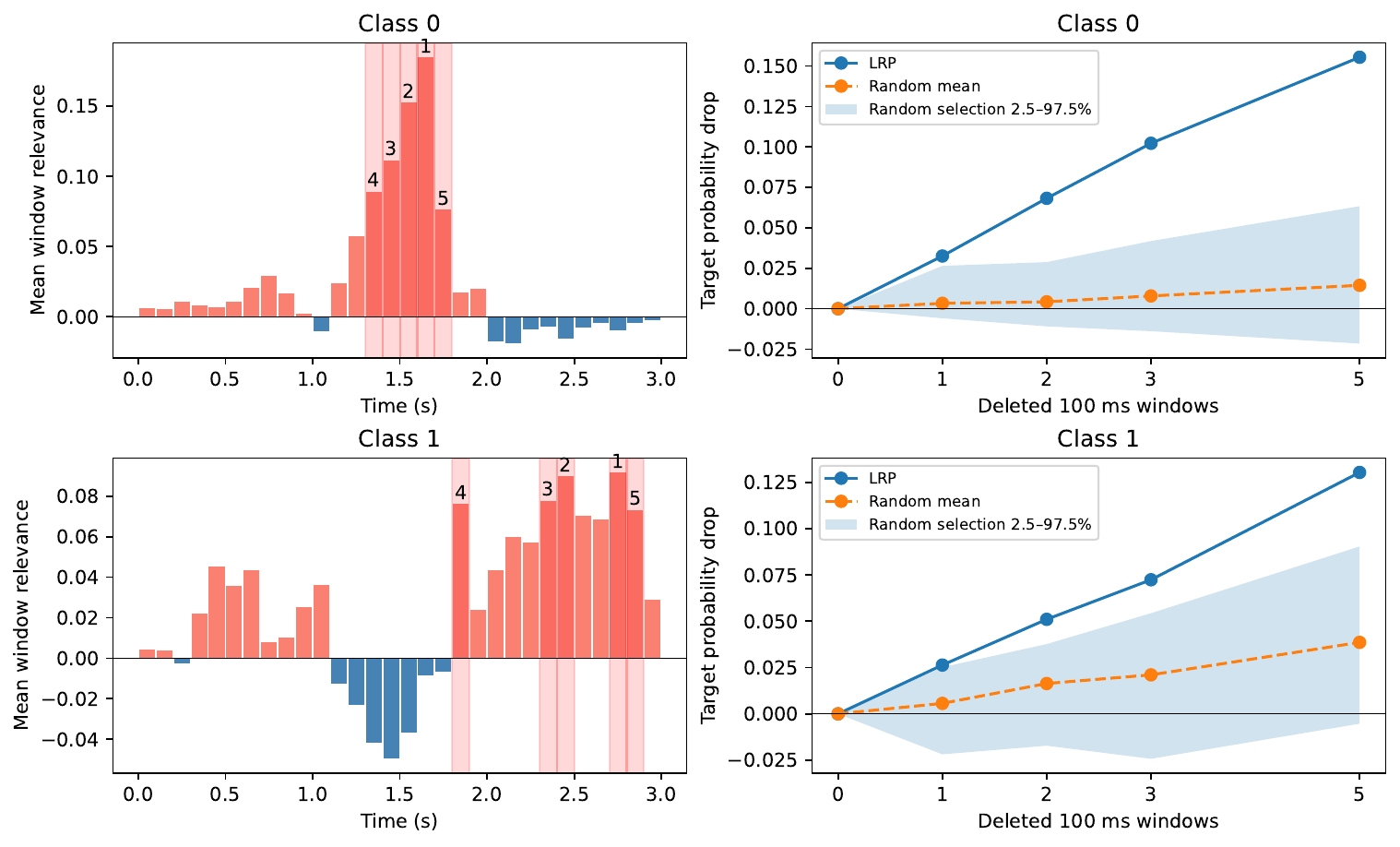}
    \caption{
    Temporal perturbation validation of LRP for low-load
    (top) and high-load (bottom) trials.
    Left: window relevance and the five highest-ranked windows.
    Right: mean target probability drops under
    LRP-guided versus random perturbation.
    Shading indicates the 2.5th-97.5th percentiles across
    20 random selections.
    }
    \label{fig:window_perturbation}
\end{figure}

\section{Model Configuration for CIFAR-10}
\label{app:model_cifar}
\paragraph{Model architecture.}
We use a lightweight Lorentz convolutional network inspired by HyperbolicCV \citep{DBLP:conf/iclr/BdeirSL24}, with explicit origin-centered exponential and logarithmic maps for evaluating relevance propagation through radial transformations. We fix the sectional curvature to $-c$ with $c=1$ and represent each hyperbolic feature by its spatial coordinates $\bm s\in\mathbb R^d$, with the time coordinate implicitly given by $t(\bm s)=\sqrt{c^{-1}+\|\bm s\|_2^2}$. Let $\bm o_L=(c^{-1/2},\bm0)$. For convenience, we denote the spatial components of the origin exponential and logarithmic maps by
\begin{equation}
\begin{aligned}
\exp_{\bm o_L,\mathrm{sp}}^{L,c}(\bm v)
&:=\left[\exp_{\bm o_L}^{L,c}\!\left((0,\bm v)\right)\right]_{\mathrm{sp}}
=\frac{\sinh(\sqrt c\,\|\bm v\|_2)}{\sqrt c\,\|\bm v\|_2}\bm v,\\
\log_{\bm o_L,\mathrm{sp}}^{L,c}(\bm s)
&:=\left[\log_{\bm o_L}^{L,c}\!\left((t(\bm s),\bm s)\right)\right]_{\mathrm{sp}}
=\frac{\operatorname{arsinh}(\sqrt c\,\|\bm s\|_2)}{\sqrt c\,\|\bm s\|_2}\bm s,
\end{aligned}
\end{equation}
where both scalar factors are defined as one at zero. Here $\bm v$ denotes the spatial component of the ambient Lorentz tangent vector $(0,\bm v)$, rather than the aligned Poincar\'e tangent coordinates used in the GRI test. These maps operate independently at each spatial location across the feature channels. The normalized RGB input $\bm x$ is first mapped to spatial Lorentz coordinates as $\bm s^{(0)}=\exp_{\bm o_L,\mathrm{sp}}^{L,c}(0.25\bm x)$. We use the factor of $0.25$ to moderate the initial tangent-space radius and the resulting hyperbolic coordinate magnitudes.

The backbone contains six $3\times3$ convolutional layers with channel widths $(32,32,64,64,128,128)$, strides $(1,1,2,1,2,1)$, and padding of one. For the flattened spatial patch $\bm p_u^{(\ell-1)}$ at location $u$, the $\ell$th convolution computes
\begin{equation}
\tau_u^{(\ell)}=\sqrt{c^{-1}+\|\bm p_u^{(\ell-1)}\|_2^2}, \qquad \bm z_u^{(\ell)}=\bm W_s^{(\ell)}\bm p_u^{(\ell-1)}+\bm w_t^{(\ell)}\tau_u^{(\ell)}+\bm b^{(\ell)},
\end{equation}
where $\bm W_s^{(\ell)}$ and $\bm w_t^{(\ell)}$ are learnable spatial and time-coordinate weights, respectively. Each convolution is followed by a tangent-space activation,
\begin{equation}
\bm s_u^{(\ell)}=\exp_{\bm o_L,\mathrm{sp}}^{L,c}\!\left(\operatorname{ReLU}\!\left(\log_{\bm o_L,\mathrm{sp}}^{L,c}(\bm z_u^{(\ell)})\right)\right).
\end{equation}
The resulting feature resolutions are $32\times32$, $16\times16$, and $8\times8$, with two convolutional layers at each resolution. 
After the final layer, we average the spatial coordinates over the $8\times8$ feature grid, apply the logarithmic map, and obtain the ten class logits through a linear classifier:
\begin{equation}
\bar{\bm s}=\frac{1}{64}\sum_{u=1}^{64}\bm s_u^{(6)}, \qquad f(\bm x)=\bm W_{\mathrm{cls}}\log_{\bm o_L,\mathrm{sp}}^{L,c}(\bar{\bm s})+\bm b_{\mathrm{cls}}.
\end{equation}
Thus, pooling is an arithmetic mean of spatial coordinates rather than a Lorentz centroid, and classification is performed in the tangent space. 
Note that this architecture is a custom lightweight variant rather than the L-ResNet18 in HyperbolicCV.

\paragraph{Training configuration.}
We split the 50000 CIFAR-10 training images into 45000 training and 5000 validation samples, and reserve the official 10,000-image test set for final evaluation. Training augmentation consists of random $32\times32$ crops with four-pixel padding and random horizontal flips. Inputs are normalized using channel-wise means $(0.4914,0.4822,0.4465)$ and standard deviations $(0.2470,0.2435,0.2616)$. We optimize cross-entropy loss for 200 epochs using AdamW with batch size 128, initial learning rate $10^{-3}$, and weight decay $10^{-4}$. The learning rate follows a cosine schedule to zero, and the global gradient norm is clipped to 5. The curvature and input scaling factor remain fixed throughout training. The checkpoint with the highest validation accuracy ($85.80\%$) is selected for test evaluation (test accuracy is $85.70\%$) and attribution experiments.

\paragraph{Relevance propagation through Lorentz convolutions.}
We treat each convolution as an affine mapping of the augmented patch coordinates $(\bm p_u,\tau_u)$ and first redistribute output relevance to the spatial and time branches using the $z^{\mathcal B}$ rule in the first layer and the $\gamma$-rule in subsequent layers. For a $\gamma$-rule layer, suppressing the layer index, the relevance assigned to the patch time coordinate is
\begin{equation}
R_{\tau_u}=\sum_j\frac{\tau_u\,\rho_\gamma(w_{t,j})}{\operatorname{stab}_{\epsilon}\!\left(\sum_i p_{u,i}\rho_\gamma(W_{s,ji})+\tau_u\rho_\gamma(w_{t,j})+b_j\right)}R_{u,j},
\end{equation}
where $\rho_\gamma(w)=w+\gamma\max(w,0)$ and $\operatorname{stab}_{\epsilon}(z)=z+\epsilon\operatorname{sign}_{+}(z)$, with $\operatorname{sign}_{+}(0)=1$. Since $\tau_u$ is determined by the spatial patch, we subsequently redistribute its relevance using squared-coordinate proportions:
\begin{equation}
R_{u,i}^{\mathrm{time}}=\frac{p_{u,i}^{2}}{\sum_m p_{u,m}^{2}}R_{\tau_u}, \qquad R_{u,i}=R_{u,i}^{\mathrm{space}}+R_{u,i}^{\mathrm{time}}.
\end{equation}
For a zero patch, where $\sum_m p_{u,m}^2=0$, we set $R_{u,i}^{\mathrm{time}}=0$ for every coordinate and record $R_{\tau_u}$ as unassigned relevance. The recorded zero-patch contribution, together with residuals from affine biases and numerical stabilization, are summed up as the total relevance conservation residual.

The first-layer $z^{\mathcal B}$ rule instead uses bounded contributions for both spatial and time coordinates, followed by the same squared-proportion redistribution. Contributions from overlapping patches are summed at each input coordinate. The fixed curvature term receives no relevance. 
This time-coordinate rule is shared by LRP-radial-all and LRP-half, which differ only in their treatment of the exponential and logarithmic radial maps.

We use a fixed $\gamma=0.25$ in convolutional layers 2--6 and $\epsilon=10^{-6}$ for denominator stabilization throughout relevance propagation, including the $\epsilon$-rule for classifier and average pooling. The first convolution uses the $z^{\mathcal B}$ rule without $\gamma$ modification, 
for which we derive a fixed bounding box from the valid RGB range $[0,1]$, accounting for normalization, input scaling, and the exponential map. 
Time-coordinate bounds are computed from these spatial bounds using the Lorentz constraint. The same box is used for all images.
These settings are identical for LRP-radial-all and LRP-half.

\paragraph{Baseline configurations.} Integrated Gradients uses a zero baseline in normalized input space, corresponding to the RGB normalization mean, and integration over 128 steps. Attributions are summed across RGB channels to obtain pixel scores.  The random baseline averages three independent pixel permutations per image.

Perturbations replace all three channels of selected pixels with the mean-color baseline. We evaluate sparsity with 5\% increment steps.
Evaluation uses 512 test images sampled uniformly without replacement. We estimate 95\% percentile bootstrap confidence intervals using 10,000 image-level resamples with replacement, sharing resampling indices across methods and metrics.

\section{Additional Figures for Experiments}
\begin{figure}[t]
    \centering
    \includegraphics[width=0.7\linewidth]{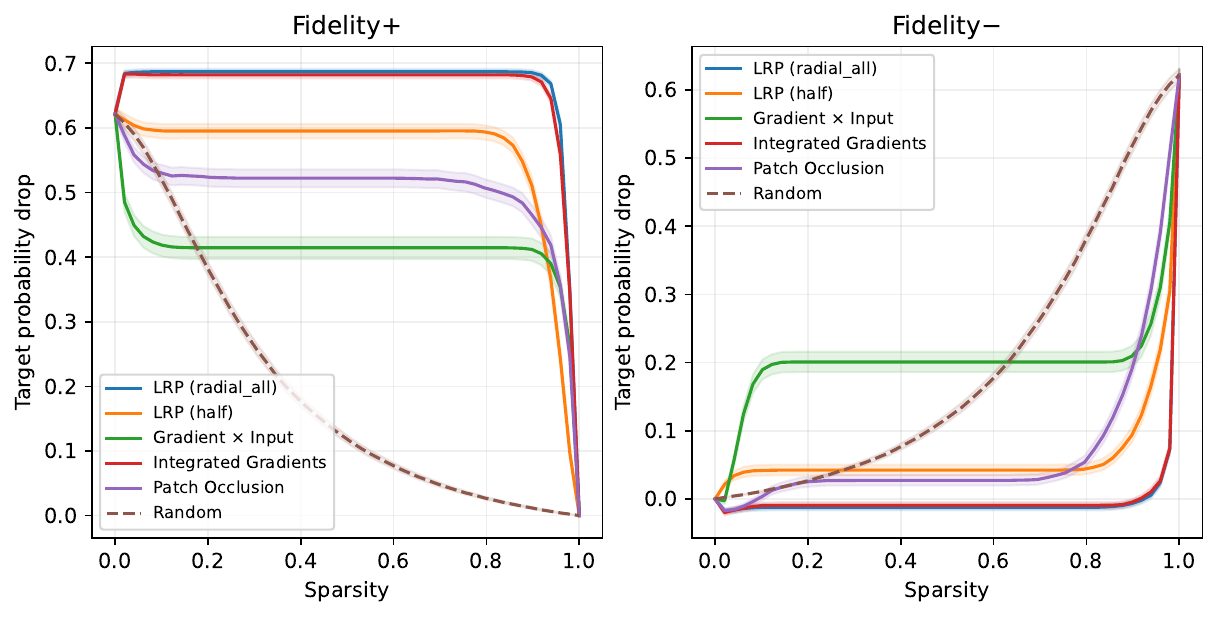}
    \caption{Fidelity-Sparsity curve for MNIST experiment with 95\% confidence interval. For Fidelity+ higher is better and for Fidelity- lower is better. Lower sparsity corresponds to a larger selected pixel set, which is removed for Fidelity+ and retained for Fidelity-.}
    \label{fig:mnist_quantitative}
\end{figure}

\begin{figure}
    \centering
    \includegraphics[width=\linewidth]{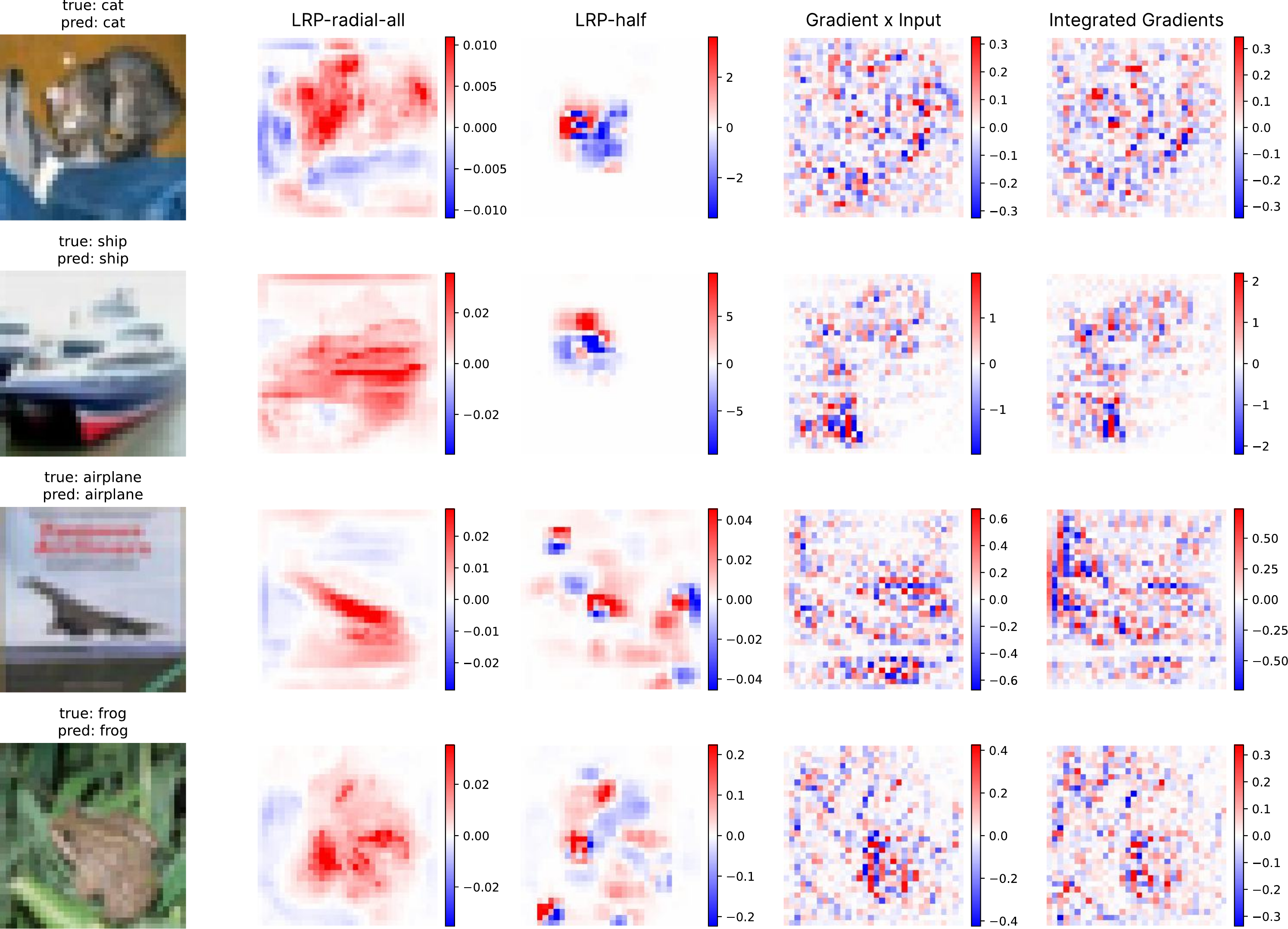}
    \caption{Qualitative attribution comparison on four CIFAR-10 examples. Columns show the input, heatmaps for LRP-radial-all, LRP-half, Gradient$\times$Input, and Integrated Gradients. RGB-channel attributions are summed, with red and blue indicating positive and negative values, respectively.}
    \label{fig:cifar_qualitative_all}
\end{figure}

\begin{figure}
    \centering
    \begin{subfigure}{0.48\textwidth}
        \centering
        \includegraphics[width=\linewidth,trim=0 3cm 0 3cm, clip]{windowed_heatmap_class_0.pdf}
        \caption{Low working-memory load class.}
    \end{subfigure}
    \hfill
    \begin{subfigure}{0.48\textwidth}
        \centering
        \includegraphics[width=\linewidth,trim=0 3cm 0 3cm, clip]{windowed_heatmap_class_1.pdf}
        \caption{High working-memory load class.}
    \end{subfigure}
    \caption{Averaged signals and relevance heatmaps for class 0 (low working-memory load) and class 1 (high working-memory load). Boxes highlight localized relevance, distributed multichannel relevance, stronger late-window relevance, and class-dependent attribution differences. Each time series stands for a sensor. Time range from left to right is 3-second with 3000 points. Red and blue indicate positive and negative relevance for the explained class score, respectively. }
    \label{fig:seeg_visualization_full}
\end{figure}

\end{document}